\documentclass[twoside,11pt]{article}

\usepackage{blindtext}
\usepackage{lastpage}
\usepackage{amsmath}
\usepackage{amssymb}
\usepackage{mathtools}
\usepackage{amsthm}
\usepackage{bbm}
\usepackage{amsfonts} 
\usepackage{nicefrac} 
\usepackage{xcolor}  
\usepackage{natbib}
\usepackage{tikz}
\usetikzlibrary{arrows.meta,calc}
\usepackage{subcaption}
\usepackage{booktabs}
\usepackage{multirow}

\usepackage[amssymb, graphicx]{jmlr2e}

\newcommand\E{\mathbb{E}}
\newcommand\R{\mathbb{R}}

\newcommand\D{\mathcal{D}}
\newcommand\Finva{ \mathcal{F}_{\mathrm{inv}} }

\newcommand\KL{\mathrm{KL}}
\newcommand\F{\mathcal{F}}
\newcommand\Q{\mathcal{Q}}
\renewcommand\P{\mathcal{P}}
\DeclareMathOperator*{\argmin}{arg\,min}
\DeclareMathOperator*{\argmax}{arg\,max}
\newcommand\Etr{E_{\mathrm{train}}}
\newcommand\risk{\mathcal{R}}

\newcommand\N{\mathcal{N}}
\newcommand\Var{\mathrm{Var}}
\newcommand\p{p}
\newcommand\Phat{\mathcal{Q}}
\newcommand\phat{q}
\newcommand\q{q}

\definecolor{darkred}{RGB}{139,0,0}

\usepackage{lastpage}

\firstpageno{1}
\begin{document}

\title{Robust Representation Learning through Explicit Environment Modeling}

\author{\name Yuli Slavutsky \email ys3938@columbia.edu \\
       \addr Department of Statistics\\Columbia University\\
           \AND
       \name David M. Blei \email david.blei@columbia.edu \\
       \addr Departments of Statistics and Computer Science\\Columbia University
       }

\editor{}

\maketitle

\begin{abstract}
We consider learning from labeled data collected across multiple environments, where the data distribution may vary across these environments. This problem is commonly approached from a causal perspective, seeking invariant representations that retain causal factors while discarding spurious ones. However, this framework assumes that the environment has no direct effect on the target. In contrast, we consider settings in which this assumption fails, but still aim to learn representations that support robust prediction on average across previously unseen environments. 
To this end, we study representations learned by explicitly modeling variation across environments and then marginalizing that variation out. 
We analyze the resulting representations and characterize when they are preferable to those learned by causal invariant-representation methods. 
We propose a concrete method based on generalized random-intercept models, a class of predictors in which such marginalization is possible, and study their generalization properties. Empirically, we show that these models outperform invariant-learning methods across a range of challenging settings.

\end{abstract}


\section{Introduction} \label{sec:intro}

We study multi-environment generalization, also known as out-of-distribution (OOD) generalization, where labeled data is collected across multiple environments and both the distribution of observations and the conditional distribution of labels given observations may vary across environments. We aim to learn representations from training environments that remain predictive for observations drawn from new environments.

Such settings arise in many applied domains in which data is collected under heterogeneous conditions. In medical imaging, for example, environments may correspond to hospitals or scanners \citep{bernal2013statistical, visscher2003mixed, payne2015revisiting}. In genomics and biology, environments often reflect laboratories, batches, or experimental conditions \citep{zhou2012genome, listgarten2010correction, zhang2010mixed, baayen2008mixed, judd2012treating}. In economics and social science, environments may correspond to countries, regions, or time periods \citep{gelman2007analysis}. Similar structure appears in education research, where data is collected across schools, districts, or countries \citep{raudenbush1999synthesizing, nye2000effects, lyu2023estimating}.
In all of these examples, models trained on data from a fixed set of environments are often deployed in new ones.

Specifically, we study learning from labeled data $(x, y)$ collected across multiple environments $e$,
\begin{equation} \label{eq:data}
\D = \{\{(x_i, y_i)\}_{i = 1}^{n_e}\}_{e \in \Etr},
\qquad
\Etr \coloneq \{e_1, \dots, e_m\}.
\end{equation}
We assume the data is sampled according to 
\begin{equation} \label{eq:Pexy}
    \prod_{j=1}^{m} \p(e_j)\, \prod_{i=1}^{n_e} \p(x_{ij} \mid e_j) \, \p(y_{ij} \mid x_{ij}, e_j).
\end{equation}
The conditional relationship between $x$ and $y$ may vary across environments, but the environments are assumed to share an underlying structure.

Our goal is to use data from the training environments to learn predictors that generalize to unseen test environments $e \notin \Etr$, drawn from $\p(e)$.
In particular, we seek a representation function $f(x)$ that supports accurate prediction of $y$ for a new observation $x$ arising in a new environment $e$, with $(e,x)$ sampled according to $\p(e)\,\p(x \mid e)$.

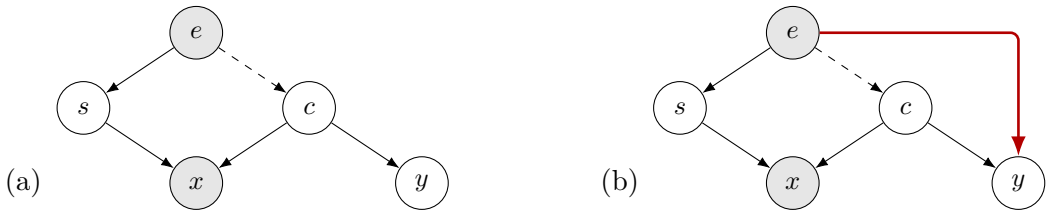
\begin{figure}[t]
  \centering
  \begin{subfigure}[t]{0.48\textwidth}
    \centering
    \phantomsubcaption\label{fig:pgm_two_panel_a}
    \begin{tikzpicture}[>=Latex]
      \tikzset{
        var/.style={circle, draw=black, thin, minimum size=7mm, inner sep=0pt, font=\small},
        grayvar/.style={var, fill=gray!20}
      }

      \node[grayvar] (e) at (0,5) {$e$};
      \node[var]     (s) at (-1.5,4) {$s$};
      \node[var]     (c) at (1.5,4) {$c$};
      \node[grayvar] (x) at (0,3) {$x$};
      \node[var]     (y) at (3,3) {$y$};

      \draw[->] (e) -- (s);
      \draw[->] (s) -- (x);
      \draw[->] (c) -- (x);
      \draw[->] (c) -- (y);
      \draw[dashed, ->] (e) -- (c);

      \node at (-2.3,3) {(\subref*{fig:pgm_two_panel_a})};
    \end{tikzpicture}
  \end{subfigure}
  \hfill
  \begin{subfigure}[t]{0.48\textwidth}
    \centering
    \phantomsubcaption\label{fig:pgm_two_panel_b}
    \begin{tikzpicture}[>=Latex]
      \tikzset{
        var/.style={circle, draw=black, thin, minimum size=7mm, inner sep=0pt, font=\small},
        grayvar/.style={var, fill=gray!20}
      }

      \node[grayvar] (e) at (0,5) {$e$};
      \node[var]     (s) at (-1.5,4) {$s$};
      \node[var]     (c) at (1.5,4) {$c$};
      \node[grayvar] (x) at (0,3) {$x$};
      \node[var]     (y) at (3,3) {$y$};

      \draw[->] (e) -- (s);
      \draw[->] (s) -- (x);
      \draw[->] (c) -- (x);
      \draw[->] (c) -- (y);
      \draw[dashed, ->] (e) -- (c);
      \draw[->, red!70!black, line width=1pt, rounded corners=3pt] (e.east) -| (y.north);

      \node at (-2.3,3) {(\subref*{fig:pgm_two_panel_b})};
    \end{tikzpicture}
  \end{subfigure}
  \caption{Probabilistic graphical model. Observed variables are colored in gray. The environment $e$ influences spurious features $s$, and may also affect causal features $c$. In (a) the label $y$ depends only on the causal features $c$, so $y \perp e \mid c$. In (b) an additional direct effect from the environment $e$ to the label $y$ is added.}
  \label{fig:pgm_two_panel}
\end{figure}

Multi-environment generalization gained considerable interest since the work of \citet{peters2016causal}. 
A prominent line of methods targets \emph{invariant representations}
\citep{rojas2018invariant,IRM,pfister2019invariant,chang2020invariant,ahuja2020invariant,mahajan2021domain,wald2021calibration,krueger2021out,lu2021invariant,lin2022bayesian}.
These methods assume that there are \emph{causal} features $c$ and \emph{spurious} features $s$, both of which may depend on the environment $e$, such that $x$ depends on both $c$ and $s$, while the label $y$ depends only on $c$; see Figure \ref{fig:pgm_two_panel_a}. 

As an example, consider predicting pneumonia diagnoses from chest X-ray images collected at multiple hospitals. Here, each environment $e$ corresponds to a hospital, the causal factors $c$ represent the underlying pathology, and the spurious factors $s$ capture hospital-specific imaging protocols. The observed input $x$ is the chest X-ray image, and $y$ is a binary indicator of pneumonia diagnosis. 

Under the model assumed by causal methods, the environment has no direct effect on the label, so any dependence of $y$ on $e$ arises only through $c$.
Accordingly, invariance-based methods seek a representation $f:\mathcal{X}\to\mathbb{R}^d$ such that $y$ is conditionally independent of $e$ given $f(x)$, that is, $y \perp e \mid f(x)$.

However, if diagnostic criteria in fact differ across hospitals, then the hospital example no longer follows the causal model assumed by invariance-seeking methods. Instead, it corresponds to the model in Figure \ref{fig:pgm_two_panel_b}. In that case, $\p(y \mid c,e) \neq \p(y \mid c)$, and no representation $f(x)$ can in general satisfy $y \perp e \mid f(x)$.

The question this paper asks is: When the environment $e$ directly affects $y$, \emph{how can we learn representations that yield good average predictive performance on observations from unseen environments?}

Formally, let 
\begin{equation} \label{eq:Ffamily}
    \F \coloneq \{f_\varphi:\mathcal{X}\to\R^d \mid  \varphi\in\Phi\}
\end{equation}
 be a family of candidate representations, and let 
 \begin{equation} \label{eq:Qfamily_0}
     \Phat \coloneq \left\{\phat_{\psi}\left(y \mid f_\varphi(x)\right) \mid \psi \in \Psi, \varphi\in\Phi \right\}
 \end{equation}
be a family of conditional distributions for $y$ given $x$, obtained by composing corresponding predictors with these representations. 
We define an \emph{environment-average risk} by 
\begin{equation} \label{eq:AvgRisk_0}
    \bar{\risk}_{\Phat}(\psi, \varphi) \coloneqq 
    \E_{\p(e)\p(x \mid e)} \left[ 
    \KL\left( \p(y \mid x, e) \, \Vert \, \phat_\psi(y \mid f_\varphi(x)) \right)
    \right].
\end{equation}
This risk measures the average discrepancy between the environment-specific Bayes predictor $\p(y \mid x,e)$ and a predictor of the form $\phat_\psi(y \mid f_\varphi(x))$, which depends only on $x$ and not on $e$.
Our goal is to find a predictor $\phat_{\psi^\star}$ and a representation $f_{\varphi^\star}$ that minimize this risk,
\begin{equation}
    (\psi^\star, \varphi^\star) \in \argmin_{\psi \in \Psi, \varphi \in \Phi} \bar{\risk}_{\Phat}(\psi, \varphi).
\end{equation}
We call such representations \emph{robust}.

As we discuss in \S \ref{sec:RI}, minimizing this objective is equivalent to maximizing the population log-likelihood over predictors of the form $\phat_\psi(y \mid f_\varphi(x))$.
The optimal predictor within the class $\Phat$ is therefore the one that approximates 
\begin{equation} \label{eq:pyx}
    \p(y \mid x)=\E_{\p(e \mid x)}[\p(y \mid x,e)].
\end{equation}
Notice that the expectation here is with respect to conditional distribution of all environments, not only the ones from the training data. 

A natural first idea is to average the environment-specific predictive distributions across the observed training environments,
$\phat_\psi(y \mid f_\varphi(x))
\approx
\sum_{j=1}^m \pi_j(x)\,\phat_\psi(y \mid f_\varphi(x),e_j)$,
where $\pi_j(x)\approx p(e_j \mid x)$.
This can reduce dependence on any single environment $e_j \in \Etr$, but it cannot recover the marginal conditional $\p(y \mid x)$ of Equation \eqref{eq:pyx}. At best, such averaging can recover $\p(y \mid x, e \in \Etr)$, contradicting our goal to generalize to new environments $e \notin \Etr$.
Recovering the full marginal conditional $\p(y \mid x)$ requires more than combining predictions from the observed environments in $\Etr$: it \emph{requires a model of the data-generating distribution} under which the environment can be marginalized out. 

In this paper, we study this idea. We first instantiate it through a generalized neural random-intercept model \citep{slavutsky2026neural}, in which the response $y$ depends on a fixed representation $f_{\varphi^\star}(x)$ that captures structure shared across environments, and on a random intercept $\gamma_e \sim \N(0,\sigma^{\star 2})$ that captures environment-specific variation. We estimate $\sigma^{\star 2}$ from the data and then marginalize over the random intercept to obtain an estimate of $p(y\mid x)$. We show that this approach outperforms invariance-seeking methods across a range of challenging tasks involving generalization to unseen environments.

More broadly, we then study predictors of the form $\phat_\psi(y \mid f_\varphi(x))$. We show that the environment-average risk of such predictors admits a decomposition into three terms: one determined by the representation $f_\rho$, one capturing the additional risk induced by the predictor $p_\psi$, and an irreducible term that depends only on the true data-generating distribution. Building on this decomposition, we show that minimizing the environment-average risk over this class of predictors yields approximately robust representations.

The rest of this paper is organized as follows. In \S \ref{sec:related} we review related work.
In \S \ref{sec:RI} we discuss generalized neural random-intercept models and how they can be used to achieve robustness through marginalization of random intercepts.
In \S \ref{sec:robust}, we study robustness through risk minimization beyond random-intercept models. In \S \ref{sec:empirical} we present empirical results. In \S \ref{sec:compare} we compare robustness through risk minimization with invariance-seeking methods. We conclude with a discussion of our results in \S \ref{sec:discuss}.

\section{Related work} \label{sec:related}

Methods for learning representations that are robust to distribution shift differ in the assumptions they make about which shifts may occur at test time.
Distributionally robust optimization (DRO) \citep{DRO, duchi2021statistics, duchi2021learning, wei2023distributionally} assumes that the test distribution may deviate from the observed training distribution within a predefined uncertainty set, and optimizes performance under the worst-case distribution in that set.
One variant is group DRO \citep{sagawa2019distributionally, piratla2021focus} which incorporates group structure that may correlate with  features that may lead to biased predictions.
When group labels are not observed, several strategies have been proposed, including reweighting high-loss examples \citep{liu2021just} and balancing class-group combinations through data subsampling \citep{idrissi2022simple}.
By contrast, our approach targets low average risk on unseen environments rather than protection against worst-case perturbations, and models environment-level heterogeneity explicitly rather than treating it as an unspecified perturbation.

Methods that explicitly consider multi-environment settings aim to learn predictors that generalize beyond the environments observed during training.
One such line of work is adversarial methods, that seek representations that do not distinguish among environments.
For example, \citet{sinha2017certifying} construct adversarial perturbations of the training distribution and optimize performance against these  perturbations. Similarly, \citet{ganin2016domain} propose domain-adversarial training, which learns features that are predictive of the label while being uninformative about the domain.
Like DRO, these approaches aim to protect against a worst-case scenario and treat the environment as a signal to be removed.
Our approach differs from these methods in the same way: it targets average performance on new environments and explicitly models environment-level variation rather than suppressing it.

Another substantial line of research approaches out-of-distribution generalization from a causal perspective.
This line of work assumes that there exists a representation for which the predictive relationship with the target is invariant across environments, and aims to retain features whose association with the target is stable, while discarding features whose predictive value is environment-specific.
This view is formalized in invariant prediction \citep{peters2016causal, peters2017elements}, and has motivated a large literature on invariant representation learning, including invariant risk minimization and its variants \citep{IRM, ahuja2020invariant, lu2021invariant, rothenhausler2021anchor, lin2022bayesian}, invariant models for transfer learning \citep{rojas2018invariant}, sequential extensions \citep{pfister2019invariant}, and related methods based on causal matching or invariant rationales \citep{mahajan2021domain, chang2020invariant}.
Closely related approaches such as risk extrapolation and calibration-based objectives also aim to improve robustness by favoring predictors whose behavior is stable across environments \citep{krueger2021out, wald2021calibration}.

Our work, however, allows invariance to fail: rather than requiring the environment to become irrelevant after conditioning on a representation, we allow it to affect the target directly and study representations that explicitly model this effect and then marginalize it out.
Our analysis in \S \ref{sec:compare} characterizes when such representations are preferable to invariant ones, and our empirical results show that this distinction matters in practice.

\section{Marginalization of environment effects} \label{sec:RI}

Our goal is to learn robust representations by minimizing the environment-average risk $\bar{\risk}(\varphi,\psi)$ in Equation \eqref{eq:AvgRisk_0}. 
This risk evaluates prediction over environments drawn from the same population of environments as the training environments.

To characterize its minimizers, we define the \emph{marginal risk}
\begin{equation} \label{eq:MargRisk}
\risk(\varphi, \psi)
\coloneqq
\E_{p(x)}
\left[
\KL\left( p(y \mid x) \Vert \phat_\psi(y \mid f_\varphi(x)) \right)
\right].
\end{equation}
The values of $\bar{\risk}(\varphi,\psi)$ and $\risk(\varphi,\psi)$ are generally different: the environment-average risk includes variation induced by unobserved environments, whereas the marginal risk does not. However, by the conditional KL decomposition\footnote{This also follows immediately from Lemma \ref{lemma:same_min} in the next section.}, this difference is independent of $(\varphi,\psi)$. Therefore, although $\bar{\risk}$ is the target risk, its minimizers can be identified by minimizing $\risk$:
\begin{equation} \label{eq:same_infima}
\argmin_{\varphi \in \Phi, \psi \in \Psi} \bar{\risk}(\varphi, \psi)
=
\argmin_{\varphi \in \Phi, \psi \in \Psi} \risk(\varphi, \psi).
\end{equation}

This indicates that the best predictor that we can recover is the true marginal conditional distribution $\p(y \mid x)$. A principled way to estimate it is to model the data-generating process through the identity
\begin{equation} \label{eq:form}
\p(y \mid x)
=
\int
\p(e \mid x)\,\p(y \mid x,e)\, de.
\end{equation}
Rather than averaging only over the observed training environments, this approach learns a model for $\p(e \mid x)$ and uses it to construct a marginalized estimator of $\p(y \mid x)$.

A flexible model class of this form is a \emph{generalized neural random-intercept model}. Below, we present this model and show how it can be used to learn robust representations.

\subsection{Neural generalized random-intercept models}

Neural generalized random-intercept models are a special case of neural generalized mixed-effects models (NGMMs) \citep{slavutsky2026neural}. Classical generalized linear mixed-effects models (GLMMs) assume that each response follows an exponential-family distribution whose natural parameter is a linear function of observed covariates and latent group-specific random effects. For this reason, GLMMs  have long been used to analyze grouped and hierarchical data (see \citet{bryk1992hierarchical} and references therein for a historical overview), but their linear predictor makes them unsuitable for learning flexible representations. NGMMs extend this model class by replacing the linear predictor with neural networks, allowing the model to capture complex relationships in the data.

Specifically, let 
\begin{equation}
\label{eq:prior}
\gamma_e \sim \N(0,\sigma^2)
\end{equation}
be environment-specific random effects, sampled independently across environments. 
A neural generalized random-intercept model assumes that the effect of the environment on the response is entirely captured by the latent variable $\gamma_e$, and that covariates are independent of $\gamma_e$.
Conditionally on $(x,\gamma_e)$, the response $y$ follows a regular one-parameter exponential family distribution
 $\q_{\beta, \varphi}(y \mid x,\gamma_e)$ satisfying 
\begin{equation}
\label{eq:RI_mean}
\E_{\q_{\beta, \varphi}}\left[y \mid x, \gamma_e\right]=
\eta^{-1}\left(\beta^\top f_\varphi(x)+\gamma_e\right),
\end{equation}
where $f_\varphi$ is a neural network with arbitrary structure.

This yields marginalized predictors of the form 
\begin{equation} \label{eq:marg_ri}
    \q_{\beta,\varphi,\sigma}(y \mid x) = \int \q_{\beta, \varphi}(y \mid x,\gamma)\,\phi_\sigma(\gamma) \, d\gamma,
\end{equation}
where $\phi_\sigma$ denotes the density of $\N(0,\sigma^2)$. 

For now, suppose that the true data-generating process is itself a random-intercept model. This leads to the following strategy for learning robust representations:
\begin{enumerate}
    \item Fit the parameters of the neural random-intercept model, $(\hat{\beta},\hat{\varphi},\hat{\sigma})$, from the data $\D$ by minimizing the empirical analogue of the marginal risk in Equation \eqref{eq:MargRisk}.
    \item For a new observation $x$ arising from an unseen environment $e$, form the prediction $\q_{\hat{\beta},\hat{\varphi},\hat{\sigma}}(y \mid x)$ with the learned parameters, as in Equation \eqref{eq:marg_ri}.
\end{enumerate}

We discuss misspecification and approximate robustness in the next section. Before that, we turn to the practical implementation of this strategy: fitting generalized neural random-intercept models to data and using them to make predictions in previously unseen environments.

\subsection{Estimation of neural generalized random-intercept models}
 \label{sec:ri_estimation}

In neural generalized random-intercept models, conditionally on $(x,\gamma_e)$, the response follows
\begin{equation}
\q_{\beta,\varphi}(y \mid x,\gamma)
=
h(y)\exp\left(
y\,\vartheta(x,\gamma;\beta,\varphi)
-
A\bigl(\vartheta(x,\gamma;\beta,\varphi)\bigr)
\right).
\end{equation}
Here $h(y)$ is the base measure, $\vartheta$ is the natural parameter of the regular one-parameter exponential family, and $A$ is the log-partition function, satisfying
\begin{equation}
\vartheta(x,\gamma;\beta,\varphi)
=
(A')^{-1}\left(\eta^{-1}\bigl(\beta^\top f_\varphi(x)+\gamma\bigr)\right).
\end{equation}
 
Given training data $\D = \{\{(x_i,y_i)\}_{i=1}^{n_e}\}_{e \in \Etr}$,
denote by
\begin{equation}
x_e \coloneqq (x_1,\dots,x_{n_e}),
\qquad
y_e \coloneqq (y_1,\dots,y_{n_e}),
\end{equation}
the covariates and responses in environment $e$, respectively.
Since minimization of the marginal risk in Equation \eqref{eq:MargRisk} is equivalent to maximization of the marginal log likelihood, 
the empirical estimates are given by
\begin{equation}
\label{eq:ri_estimator}
(\hat\beta,\hat\varphi,\hat\sigma)
\in
\argmax_{\beta,\varphi,\sigma} \mathcal{L}_{\Q}(\theta), 
\qquad 
\mathcal{L}_{\Q}(\theta) \coloneqq
\sum_{e \in \Etr}
\log \q_{\beta,\varphi,\sigma}(y_e \mid x_e),
\end{equation}
where, by conditional independence within each environment,
\begin{equation}
\label{eq:marginal_likelihood_env}
\q_{\beta,\varphi,\sigma}(y_e \mid x_e)
=
\int_{\R}
\left(
\prod_{i=1}^{n_e}
\q_{\beta,\varphi}(y_i \mid x_i,\gamma)
\right)
\phi_\sigma(\gamma)\, d\gamma.
\end{equation}

For general exponential-family responses, the integral in Equation \eqref{eq:marginal_likelihood_env} does not admit a closed form.
However, in the random-intercept case it is reduced to solving a one-dimensional ordinary differential equation.
This yields a differentiable expression of marginal log likelihood, which can be optimized by stochastic gradient ascent.
See \S 2.2 in \citep{slavutsky2026neural} for the general derivation, Appendix A for a random-intercept logistic illustration, and \S 2.3 for the corresponding stochastic optimization procedure.

\subsection{Prediction} \label{sec:ri_prediction}

Once the model is fit, it can be used to predict responses for datapoints from previously unseen environments $e' \notin \Etr$.
Since no labeled data from $e'$ is available at prediction time, the estimated posterior of its random intercept coincides with the estimated prior,
\begin{equation}
\gamma_{e'} \sim \N(0,\hat\sigma^2).
\end{equation}
The predictive distribution is then generally obtained by replacing the parameters $(\beta,\varphi,\sigma)$ in Equation \eqref{eq:marg_ri} by their estimates $(\hat\beta,\hat\varphi,\hat\sigma)$.

However, in two special cases, Gaussian responses and binary responses with a symmetric inverse link, the Bayes-optimal prediction for a new environment depends only on the fixed component $\beta^\top f_\varphi(x)$.
In these cases, prediction in a new environment can be performed  without computing the marginalization.

The Gaussian case is immediate: under identity link we have
\begin{equation}
\E_{\q_{\beta,\varphi}}[y \mid x,\gamma]
=
\beta^\top f_\varphi(x)+\gamma,
\end{equation}
so for a new environment, since $\E[\gamma]=0$, we have
\begin{equation}
\E_{\q_{\beta,\varphi,\sigma}}[y \mid x]
=
\int \left(\beta^\top f_\varphi(x)+\gamma\right)\phi_\sigma(\gamma)\, d\gamma
=
\beta^\top f_\varphi(x).
\end{equation}

Binary responses are addressed by the following proposition.

\begin{proposition}[Prediction in new environments, Proposition 2 in \citeauthor{slavutsky2026neural}] \label{prop:binary_new_envs}
Let $y \in \{0,1\}$, and suppose the inverse link $\eta^{-1}$ is strictly increasing and satisfies $1-\eta^{-1}(u)=\eta^{-1}(-u)$
for all $u \in \R$.
Then
\begin{equation}
\hat y(x)
\coloneqq
\arg\max_{c\in\{0,1\}}
\q_{\beta,\varphi,\sigma}(y=c \mid x)
\end{equation}
depends only on the fixed component $\beta^\top f_\varphi(x)$.
In particular, $\hat y(x)=1
\Longleftrightarrow
\beta^\top f_\varphi(x)\geq 0$.
\end{proposition}

The proposition above shows that in a binary model, deployment in a new environment does not require marginalization.

While these properties make the generalized neural random-intercept model a convenient choice, the benefits of learning predictors by marginalizing over environment effects extend beyond this model. We discuss these benefits in the next section.

\section {Robustness through risk minimization} \label{sec:robust}

We now extend our discussion more broadly to models of the form $\phat_\psi(y \mid f_\varphi(x))$.
Recall the family of candidate representations $\Phat$ in Equation \eqref{eq:Qfamily_0} and the environment-average risk in Equation \eqref{eq:AvgRisk_0}. For convenience, here we collect the representation parameters $\varphi\in \Phi$ and the predictor parameters $\psi \in \Psi$ into a single parameter $\theta \in \Theta$ where $\theta=(\psi,\varphi)$ and $\Theta \coloneq \Psi \times \Phi$.

We begin our analysis by showing that the environment-average risk admits the following decomposition.

\begin{lemma} \label{lemma:same_min}
For every $\theta \in \Theta$,
\begin{equation}
\bar{\risk}_{\Phat}(\theta)
=
\risk_{\Phat}(\theta)
+
\Delta(\bar{\risk}, \risk), \qquad
\Delta(\bar{\risk}, \risk) \coloneqq
\E_{\p(e)\p(x \mid e)}\left[
\KL\left(
\p(y \mid x,e) \, \Vert \, 
\p(y \mid x)
\right)
\right].
\end{equation}
\end{lemma}

See Appendix \ref{sup:proof_same_min} for proof. 
Here, the first term is precisely the marginal risk of Equation \eqref{eq:MargRisk}, which measures the discrepancy between the model $\phat_\theta(y \mid x)$ and the true marginal conditional distribution $p(y \mid x)$. The second term $\Delta(\bar{\risk}, \risk)$ is irreducible: it does not depend on the predictor $\phat_\theta(y \mid x)$ and therefore does not depend on the learned representation. It measures the average residual variation in the conditional distribution of $y$ given $x$ across environments, which persists because prediction in a new environment is based only on $x$, while $\p(y \mid x,e)$ varies with $e$.

Whenever $\phat_\theta(y \mid x)$ is correctly specified, that is, there exists $\theta \in \Theta$ such that
$\phat_{\theta^\star}(y \mid x)=\p(y \mid x)$,
we have $\risk_{\Phat}(\theta^\star)=0$ and $\bar{\risk}_{\Phat}(\theta^\star)=\Delta(\bar{\risk},\risk)$. 
Hence, under correct specification, the minimum achievable environment-average risk is exactly the irreducible term. 
In particular, for $\theta^\star=(\psi^\star,\varphi^\star)$  that satisfies $\theta^\star \in \argmin_{\theta \in \Theta} \risk_{\Phat}(\theta^\star)$, the representation $f_{\varphi^\star}$ is \emph{robust}.

Moreover, as the following corollary shows, for any predictor in $\Phat$, the excess environment-average risk relative to the true data-generating process is exactly its marginal risk.

\begin{corollary}[Excess environment-average risk] \label{cor:excess_env_risk}
Let $\bar{\risk}^\star_{\P}$ and $\risk^\star_{\P}$ denote the environment-average and marginal risks of the true data-generating process $\p$. Then, for any predictor $\phat_\theta \in \Phat$,
\begin{equation}
\bar{\risk}_{\Phat}(\theta)-\bar{\risk}^\star_{\P}
=
\risk_{\Phat}(\theta).
\end{equation}
\end{corollary}

\begin{proof}
By Lemma \ref{lemma:same_min}, the environment-average risk and the marginal risk differ by the same additive term for every predictor. In particular, for any $\theta \in \Theta$ we have $\bar{\risk}_{\Phat}(\theta)-\risk_{\Phat}(\theta) = \bar{\risk}^\star_{\P}-\risk^\star_{\P}$. 
Since,  by definition, $\risk^\star_{\P}=0$,  it follows that
$\bar{\risk}_{\Phat}(\theta)-\risk_{\Phat}(\theta)
= \bar{\risk}^\star_{\P}$.
\end{proof}

However, since the true data-generating mechanism is unknown, a model  $\phat_\theta(y \mid x)$ will be in general misspecified. In such case exact robustness may not be attainable; see Figure \ref{fig:RI} for an illustration in the generalized neural random-intercept setting.

\begin{figure} [ht]
    \centering
    \includegraphics[width=0.8\linewidth]{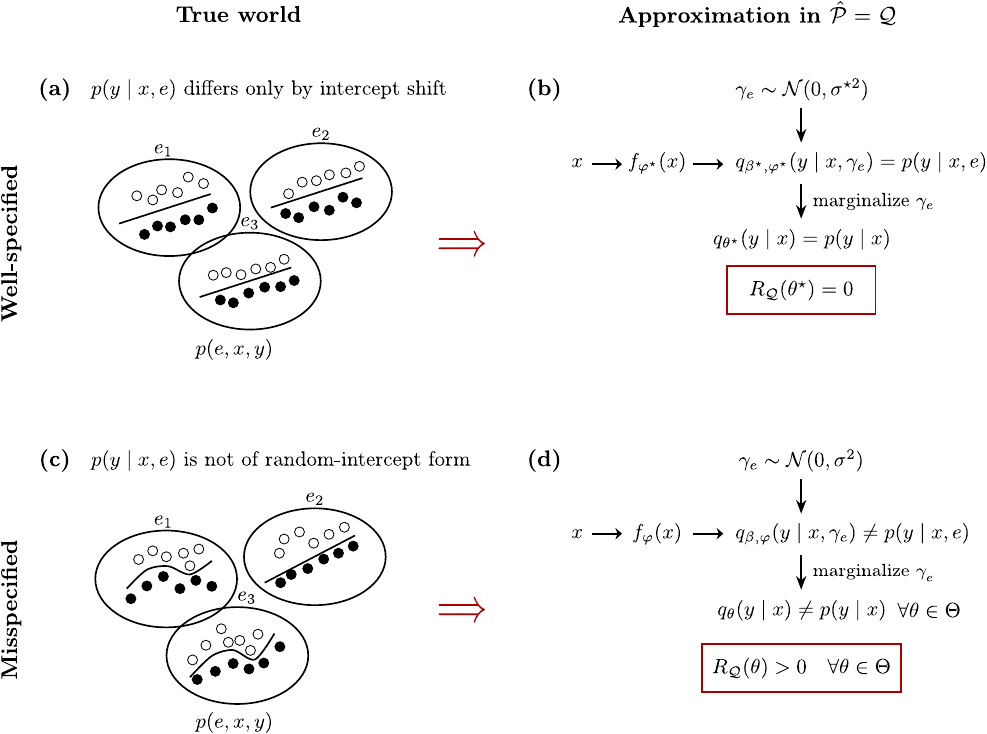}
    \caption{Schematic comparison between the true data-generating process and approximation within the random-intercept family.
Left: the true joint distribution $\p(e,x,y)$ induces environment-specific conditionals $\p(y \mid x,e)$, shown here as labeled data across environments.
Top: the conditionals differ only by intercept shift.
Bottom: the environment effect is not of random-intercept form.
Right: in both cases we approximate using predictors in $\Q$, obtained by mapping $x$ to $f(x)$, introducing a latent intercept $\gamma_e$, and marginalizing it out.
When the model is correctly specified, the induced marginal predictor recovers $\p(y \mid x)$; under misspecification, it generally does not.}
    \label{fig:RI}
\end{figure}

To address this, we relax our goal and aim to find a representation $f_\varphi \in \F$ for which there exists a predictor $\phat_\psi$ whose environment-average risk is within $\varepsilon$ of the infimum over $\Theta = \Psi \times \Phi$. The following definition makes this precise.

\begin{definition}[$\epsilon$-robust representation] \label{def:robust}
Let $\Phat$ be a family of conditional distributions as in Equation \eqref{eq:Qfamily_0}, indexed by the representation parameters $\varphi$  and the remaining model parameters $\psi$. 
For $\varepsilon \geq 0$, a representation $f_{\varphi^\star}$ is called $\varepsilon$-robust if there exists $\psi^\star \in \Psi$ such that  $(\psi^\star,\varphi^\star)$ satisfies
\begin{equation}
    \bar{\risk}_{\Phat}\big(\psi^\star,\varphi^\star\big) \leq \bar{\risk}_{\Phat}\big(\psi,\varphi\big) + \varepsilon
    \qquad
    \forall \varphi \in \Phi, \psi \in \Psi .
\end{equation}
We denote the subset of $\varepsilon$-robust representations in $\F$ by $\F_\varepsilon$.
\end{definition}

This definition extends the previous one. In particular, if the minimum of $\bar{\risk}_{\Phat}$ is attained, then a \emph{robust} representation is recovered at $\varepsilon=0$.

Under misspecification, typically $\risk_{\Phat}(\theta^\star)>0$. 
The next proposition provides a decomposition of the marginal risk $\risk_{\Phat}$ itself.

\begin{proposition}[Decomposition of the marginal risk] \label{prop:decomposition}
Let $\Phat$ be a family of predictors as in Equation \eqref{eq:Qfamily_0}
parameterized by $\theta \in \Theta \coloneq \Psi \times \Phi$, and let $f_\varphi:\mathcal X \to \mathbb R^d$ denote the corresponding representation for $\varphi \in \Phi$. Then, for any $\theta \in \Theta$,
\begin{equation} \label{eq:risk_decomp}
\risk_{\Phat}(\theta)
=
I\left(y;x \mid f_\varphi(x)\right)
+
\Delta_{\Phat \mid f},
\end{equation}
where
\begin{equation}
I\left(y;x \mid f_\varphi(x)\right)
=
\E_{\p(x)}
\left[
\KL\left(
\p(y \mid x)
\, \Vert \,
\p(y \mid f_\varphi(x))
\right)
\right]
\end{equation}
is the risk induced by the representation $f_\varphi$, and
\begin{align}
% \!\!\! 
\Delta_{\Phat \mid f}
 \coloneqq
\E_{\p(x)}
\left[
\KL\left(
\p(y \mid f_\varphi(x))
\, \Vert \,
\phat_{\theta}(y \mid x)
\right)
\right]
\end{align}
is the additional risk induced by predictor, given the representation.
\end{proposition}

The proof is provided in Appendix \ref{sup:proof_decomposition}. 
Proposition \ref{prop:decomposition} shows that the marginal risk separates into two components: a representation term, determined by how much predictive information is lost when $x$ is compressed to $f_\varphi(x)$, and an approximation term, determined by the restriction of the predictor class conditional on that representation.

The first term, $I\left(y;x \mid f_\varphi(x)\right)$,
 measures the predictive information in $x$ about $y$ that is not retained by the representation $f_\varphi(x)$. Since both terms on the right-hand side of Equation \eqref{eq:risk_decomp} are nonnegative, Proposition \ref{prop:decomposition} immediately implies that
\begin{equation}
I\left(y;x \mid f_\varphi(x)\right)
\leq
\risk_{\Phat}(\theta).
\end{equation}
Therefore, if $\theta=(\psi,\varphi)$ induces small marginal risk, then the corresponding representation $f_\varphi$ cannot discard much of the information in $x$ that is relevant for predicting $y$.

The second term, $\Delta_{\Phat \mid f}$, measures the additional risk induced by restricting the predictor, conditional on the representation $f_\varphi(x)$, to the model class $\Phat$. In other words, it quantifies the discrepancy between the best predictor based only on the representation, namely $\p(y \mid f_\varphi(x))$, and the best predictor within the model class $\Phat$.

In particular, in generalized neural random-intercept models,
$\Delta_{\Q \mid f}$ admits a direct interpretation in terms of environment-level heterogeneity: if a predictor in this family achieves small risk with a small $\sigma$, this suggests that little residual environment-level variation remains after accounting for the fixed component $\beta^\top f_\varphi(x)$. In contrast, if small risk is attained only with a larger fitted $\sigma$, then, within the random-intercept approximation, more of the remaining variation is attributed to an environment-level shift. 

Combined with Corollary \ref{cor:excess_env_risk}, this shows that, under mild misspecification, a predictor $\p_{\theta^\star} \in \Phat$ with small marginal risk yields a representation that remains informative beyond the family itself. Specifically, for $\varepsilon=\risk_{\Phat}(\theta^\star)$ the corresponding representation $f_{\varphi^\star}$ is $\varepsilon$-robust.

\paragraph{Empirical environment-average risk}

So far, we analyzed the achievable risks and the value of $\varepsilon$ for which a representation is $\varepsilon$-robust, at the population level. In practice, however, the distributions $\p(e)$, $\p(x \mid e)$, and $\p(y \mid x,e)$ are not observed directly. Instead, we observe a finite dataset $\D$ of the form given in Equation \eqref{eq:data}. This introduces an additional source of error. 

Given access only to a finite dataset $\D$, we minimize an empirical objective in place of the population marginal risk. Let $\hat{\theta}$
denote the resulting estimator, and suppose that $\theta^\star$ satisfies $\bar{\risk}_{\Phat}(\theta^\star) \leq \bar{\risk}_{\Phat}(\theta) + \varepsilon_0$ for all $\theta \in \Theta$.
Then Lemma \ref{lemma:same_min} allows us to rewrite the environment-average risk as
\begin{equation}
\bar{\risk}_{\Phat}(\hat{\theta})
\quad = \quad
\underbrace{
\Delta(\bar{\risk},\risk)
}_{\text{irreducible error}}
 +
\underbrace{
(\risk_{\Phat}(\hat{\theta}) - \risk_{\Phat}(\theta^\star))
}_{\text{finite-sample estimation error}}
+
\underbrace{
\risk_{\Phat}(\theta^\star).
}_{\text{$\Phat$ specification error}}
\end{equation}

As a result, in practice, the best robustness achievable in $\Phat$ is with  $\varepsilon = \risk_{\Phat}(\hat{\theta})$, which is composed of the finite-sample estimation error $(\risk_{\Phat}(\hat{\theta}) - \risk_{\Phat}(\theta^\star))$, and the specification error $\risk_{\Phat}(\theta^\star)$ of the chosen family $\Phat$. It immediately follows that $\varepsilon \geq \risk_{\Phat}(\theta^\star) - \varepsilon_0$, providing a lower bound on the achievable range of $\varepsilon$, allowing to assess the fit.

We next study the risk $\bar{\risk}_{\Phat}(\hat{\theta})$ empirically. In \S \ref{sec:simulation}, we examine both well-specified and misspecified settings, comparing $\bar{\risk}_{\Phat}(\hat{\theta})$ with $\risk_{\Phat}(\theta^\star)$ in each case. This comparison allows us to assess both the predictive performance of the learned representations and the specification error of the proposed model.

\section{Empirical studies} \label{sec:empirical}

In all experiments, we consider a multi-environment out-of-distribution setting: training data consists of labeled examples collected from several observed environments, and evaluation is performed on held-out environments not used during training. The goal is to learn robust predictors that generalize well to new environments. We consider both synthetic and real-data settings, and include binary as well as Gaussian responses.

Across experiments, we compare the NGMM random-intercept model to empirical risk minimization (ERM), which ignores environment structure, and to invariance-seeking methods: invariant risk minimization (IRM) \citep{IRM}, variance risk extrapolation (VaREx) \citep{krueger2021out}, and, in binary experiments, calibration loss over environments (CLoVE) \citep{wald2021calibration}, which is designed for classification.

In every experiment, all methods use the same underlying predictor architecture. Hyperparameters for IRM, VaREx, and CLoVE are selected separately for each experiment by grid search.  
Additional experimental details are provided in Appendix \ref{sup:details}.

\subsection{Tradeoff simulation} \label{sec:simulation}

We begin by evaluating the environment-average risk in a controlled Gaussian simulation. In this experiment, the response depends on both a causal latent variable and the environment directly.

We simulate environments
\begin{equation}
e_j \sim \N(0,\sigma^2_e),
\end{equation}
and, for each observation $i$ in environment $j$, define the causal latent variable by
\begin{equation}
c_{ij} \mid e_j \sim \N(e_j,\sigma_c^2).
\end{equation}
We then generate latent noise variables 
\begin{equation}
    u_{ij} \sim \N(0, \sigma^2_u), \qquad \epsilon_{ij} \sim \N(0, \sigma^2_\epsilon)
\end{equation}
and set 
\begin{align}
    s_{ij} & = \alpha e_j + u_{ij}, \\ 
    x_{ij} & = (c_{ij}, s_{ij}).
\end{align}
Finally, we consider a well-specified setting, and a misspecified one:
\begin{align}
    y_{ij} & = c_{ij} + \alpha e_j + \epsilon_{ij}, \label{eq:spec} \quad \alpha \geq 0 \\ 
    y_{ij} & = c_{ij} + \alpha e_j  + (1-\alpha) e_j c_{ij} + \epsilon_{ij}, \quad \alpha\in [0,1]. \label{eq:misspec}
\end{align}

\paragraph{Well specified setting}
In the model in Equation \eqref{eq:spec}, the spurious variable $s_{ij}$ is a noisy version of $\alpha e_j$ and therefore for $\alpha \neq 0$ it carries additional information regarding $y_{ij}$ when $c_{ij}$ is known.
This yields a model where
\begin{equation}
    y_{ij} \mid c_{ij}, e_j \sim \N(c_{ij} + \alpha e_j, \sigma^2_\epsilon)
\end{equation}
so $y \perp e \mid c$ only if $\alpha=0$. 
Similarly, 
\begin{equation}
    y_{ij} \mid s_{ij}, e_j \sim \N ((1 + \alpha)e_j, \, \sigma^2_c + \sigma^2_\epsilon)
\end{equation}
which depends on $e_j$ for any $\alpha \geq 0$.
The optimal risk admits the following closed form
\begin{equation}
\bar{\risk}^\star
=
\frac{1}{2}
\log\left(
1+\frac{\alpha^2 V_\alpha}{\sigma_{\epsilon}^2}
\right),
\qquad
V_\alpha
=
\left(
\frac{1}{\sigma_e^2}
+
\frac{1}{\sigma_c^2}
+
\frac{\alpha^2}{\sigma_u^2}
\right)^{-1}.
\end{equation}
For a complete derivation see Appendix \ref{sec:spec_deriv}.

\paragraph{Misspecified setting}
In the model in Equation \eqref{eq:misspec}, we have
\begin{equation}
y_{ij}
=
c_{ij}
+
\alpha e_j
+
(1-\alpha)e_j c_{ij}
+
\epsilon_{ij}
=
c_{ij}
+
\bigl(\alpha + (1-\alpha)c_{ij}\bigr)e_j
+
\epsilon_{ij}.
\end{equation}
The posterior distribution of $e_j$ is
\begin{equation}
e_j \mid c_{ij}, s_{ij}
\sim
\mathcal N\bigl(m(c_{ij}, s_{ij}), V_\alpha\bigr),
\end{equation}
where
\begin{equation}
m(c,s)
=
V_\alpha
\left(
\frac{c}{\sigma_c^2}
+
\frac{\alpha s}{\sigma_u^2}
\right),
\qquad
V_\alpha
=
\left(
\frac{1}{\sigma_e^2}
+
\frac{1}{\sigma_c^2}
+
\frac{\alpha^2}{\sigma_u^2}
\right)^{-1}.
\end{equation}
Consequently,
\begin{equation}
y_{ij} \mid c_{ij}, s_{ij}
\sim
\mathcal N\Bigl(
c_{ij} + \bigl(\alpha + (1-\alpha)c_{ij}\bigr) m(c_{ij}, s_{ij}),
\sigma_\epsilon^2 + \bigl(\alpha + (1-\alpha)c_{ij}\bigr)^2 V_\alpha
\Bigr),
\end{equation}
and the Bayes-optimal environment-average risk is
\begin{equation}
\bar{\risk}^\star
=
\frac{1}{2}
\E_{\p(c)}
\left[
\log\left(
1+\frac{\bigl(\alpha + (1-\alpha)c\bigr)^2 V_\alpha}{\sigma_{\epsilon}^2}
\right)
\right].
\end{equation}
For a complete derivation see Appendix \ref{sec:misspec_deriv}.
Additional details for both simulations are provided in Appendix \ref{sup:tradeoff_sim}.

\paragraph{Results}
The results in Figure \ref{fig:sim} show that NGMM consistently performs best among the learned methods in both regimes. In the well-specified setting, where the environment effect is exactly of random-intercept form, NGMM nearly coincides with the Bayes-optimal risk across the full range of $\alpha$, while ERM, VaREx, and especially IRM incur substantially larger environment-average risk as $\alpha$ increases. 

In the misspecified setting, the random-intercept model is no longer exact, since the environment effect depends on both $e$ and $c$. Even so, NGMM yields lower environment-average risk than the invariance-based and ERM baselines for all values of $\alpha$. The NGMM risk decreases as $\alpha$ grows, approaching the optimal curve for $\alpha>0.5$. At $\alpha=1$ the misspecified model reduces back to the additive random-intercept case. For all values of $\alpha$, ERM, VaREx, and IRM remain far from optimal risk.

\begin{figure} [ht]
    \centering
    \includegraphics[width=1.0\linewidth]{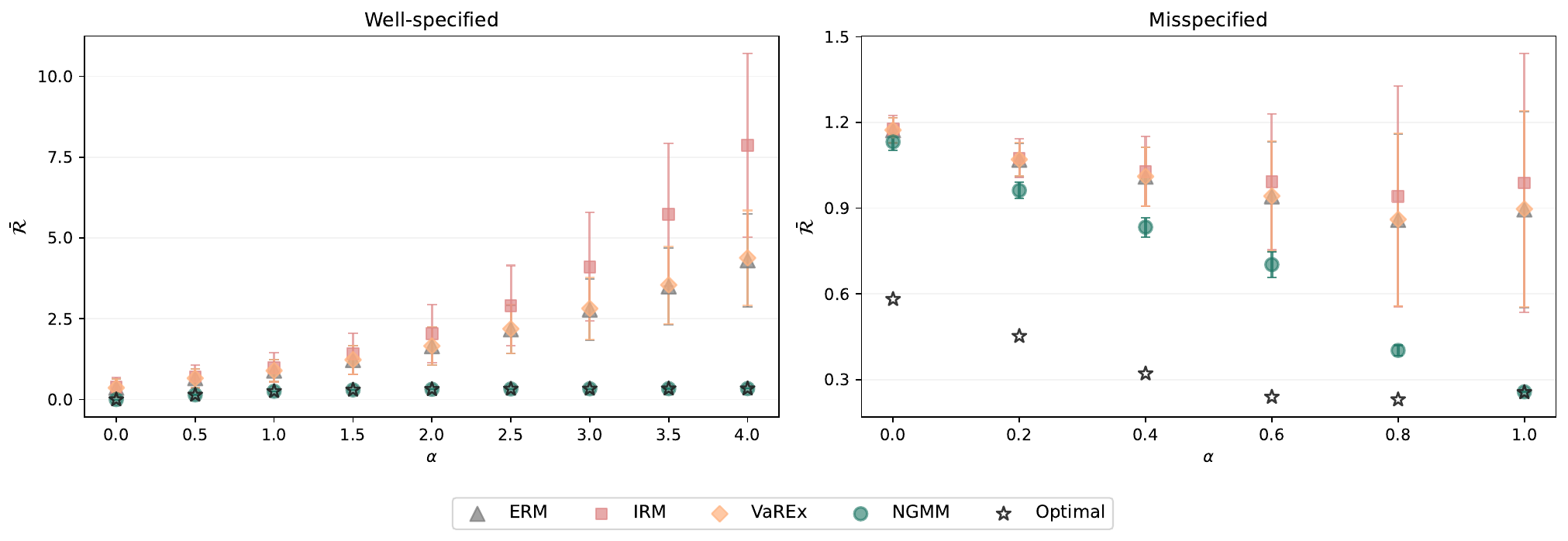}
    \caption{Environment-average risk in the tradeoff simulation for the well-specified and misspecified settings. In the well-specified case, NGMM nearly overlaps with the Bayes-optimal predictor across all values of $\alpha$. In the misspecified case environment effect is no longer of random-intercept form. NGMM still achieves the lowest risk among the learned methods and moves closer to the optimum as $\alpha$ increases and the model approaches an additive-intercept setting.}
    \label{fig:sim}
\end{figure}

\subsection{Colored MNIST}
Following the experiment in \citep{IRM}, we evaluate all methods on a controlled variant of Colored MNIST. Each observation consists of a digit image, a binary color attribute, and a binary label. The construction proceeds in three steps.
We draw digits uniformly from the MNIST dataset and define a binary parity variable $d\in\{0,1\}$, where $d=0$ denotes even digits, and $d=1$ odd digits. Each digit is assigned a binary color variable $c\in\{0,1\}$ according to a fixed parity-dependent distribution, 
\begin{equation}
    \p(c=1 \mid d=0)= \frac{5}{8}, \quad \p(c=1 \mid d=1)= \frac{3}{8}.
\end{equation}
Therefore, color is correlated with digit parity but is never deterministically determined by it. This mechanism is shared across all environments.

Given parity $d$ and color $c$, the binary label $y\in\{0,1\}$ is generated according to environment-specific probabilities:
\begin{align}
    & \p(y=1 \mid d=0, c=1)=\alpha, \quad (y=1 \mid d=0, c=0)=\beta \\
    & \p(y=1 \mid d=1, c=1)=\gamma, \quad (y=1 \mid d=1, c=0)=\delta. \notag
\end{align}

We consider two regimes: a \emph{spurious regime}, where the label depends only on digit parity and not on color conditional on parity. This is enforced by setting $\alpha=\beta$ and $\gamma=\delta$. Under this regime, $y \perp c \mid d$, so color has no causal effect on the label. Nevertheless, because color is correlated with parity, color and label may still be marginally correlated. This regime satisfies the standard assumption that a subset of features (the digit content) is invariant across environments, while other features (color) are spurious.

Additionally, we consider a \emph{causal regime}, where the conditional independence above is violated, and the label depends on color within each parity class. We parameterize environments as
\begin{equation}
    (\alpha, \beta, \gamma, \delta) = (a + r, a-r, b +r, b-r).
\end{equation}
Here, $a$ and $b$ control the baseline effect of parity, and $r$ controls the strength and direction of the color effect. 
For any $r \neq 0$, 
\begin{equation}
    \p(y=1 \mid d, c=1) \neq \p(y=1 \mid d, c=0),
\end{equation}
so color has a direct influence on the label.

Training environments use positive values of $r$ with varying magnitude $r\in \{0.10, 0.15, 0.20\}$, inducing environment-dependent color–label association.
We consider the following test environments: (i) an environment identical to train environment with $r=0.15$, (ii) environments where the sign of $r$ is reversed, corresponding to a reversal of the color–label relationship with $r\in \{-0.10, -0.15, -0.20\}$, and (iii) an environment with $r=0$, where color has no effect.

This construction includes environments in which color is predictive within training data but does not correspond to a stable mechanism across environments.

\paragraph{Results}
The results in Figure \ref{fig:cmnist} show that NGMM achieves the smallest discrepancy from the Bayes-optimal predictor in every test environment. 
This holds both in the invariant setting, where color is correlated with the label but has no effect on it conditional on digit parity, and in the causal setting, where color directly affects the label and this effect varies across environments.

In the invariant setting, robust prediction requires avoiding reliance on color because its association with the label is not stable across environments.
In the causal setting, by contrast, color cannot be discarded as it remains  predictive of the label, even though its effect changes with the environment. NGMM performs better in both regimes because it models residual environment-level variation rather than forcing the representation to remove all environment-dependent information. As a result, it can avoid overusing unstable color information in the invariant case, while still retaining useful color information in the causal case.

\begin{figure} [ht]
    \centering
    \includegraphics[width=0.85\linewidth]{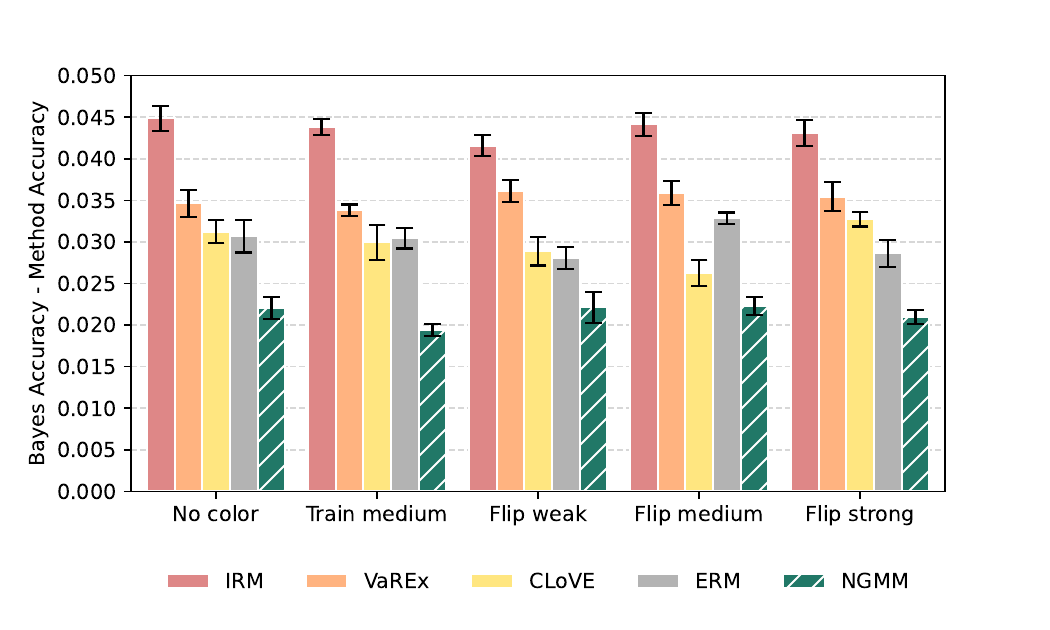}
    \caption{Results for the Colored MNIST experiment, reported as the gap between each method’s accuracy and the Bayes-optimal accuracy in per test environment. Across all settings, including the in-distribution no-color environment, and environments where the sign of the color effect is reversed, NGMM achieves the smallest discrepancy from the Bayes predictor.}
    \label{fig:cmnist}
\end{figure}

\subsection{OGB-MolPCBA}

The OGB-MolPCBA dataset provides data for real-world molecular prediction tasks. Each datapoint is a molecular graph, with atoms as nodes and bonds as edges, and the original benchmark provides 128 assays per molecule. Each assay defines a binary task: predicting from molecular structure whether the molecule is active in the corresponding laboratory screening test. In our experiment, the environments are molecular scaffold groups, which group molecules according to their core structural framework. 

We select one assay at a time and focus on assays with at least 1000 observations in the training split, and tasks in which the proportion of positive examples is at least $5\%$. This results in 4 candidate assays: PCBA-588342, PCBA-686979, PCBA-504332, and PCBA-686978. 
For all methods we use a small graph neural network with atom and bond encoders, followed by two GINE message-passing layers. 

\paragraph{Results} 
The results in Table 1 show that NGMM achieves the best performance across all four OGB-MolPCBA assays. NGMM attains the highest accuracy and the lowest NLL, Brier score, and ECE in every case, indicating improvements both in classification and in the quality of the predicted probabilities. Since the environments are molecular scaffold groups, these results suggest that explicitly modeling environment-level variation yields more robust prediction across structural shifts compared to invariance-based methods.

\begin{table}[ht]
\centering
\resizebox{\textwidth}{!}{%
\begin{tabular}{llcccc}
\toprule
Dataset & Method & Accuracy & NLL & Brier & ECE \\
\midrule
\multirow{5}{*}{PCBA-588342} & ERM & 0.810 $\pm$ 0.188 & 0.947 $\pm$ 1.151 & 0.167 $\pm$ 0.169 & 0.169 $\pm$ 0.182 \\
 & IRM & 0.862 $\pm$ 0.070 & 0.604 $\pm$ 0.289 & 0.118 $\pm$ 0.057 & 0.111 $\pm$ 0.067 \\
 & CLoVE & 0.884 $\pm$ 0.031 & 0.555 $\pm$ 0.158 & 0.103 $\pm$ 0.023 & 0.094 $\pm$ 0.028 \\
 & VaREx & 0.806 $\pm$ 0.163 & 0.708 $\pm$ 0.475 & 0.160 $\pm$ 0.124 & 0.162 $\pm$ 0.147 \\
 & NGMM & \textbf{0.908 $\pm$ 0.007} & \textbf{0.319 $\pm$ 0.029} & \textbf{0.082 $\pm$ 0.004} & \textbf{0.035 $\pm$ 0.016} \\
\midrule
\multirow{5}{*}{PCBA-686979} & ERM & 0.790 $\pm$ 0.036 & 1.187 $\pm$ 0.247 & 0.192 $\pm$ 0.023 & 0.183 $\pm$ 0.024 \\
 & IRM & 0.721 $\pm$ 0.151 & 1.636 $\pm$ 0.961 & 0.247 $\pm$ 0.126 & 0.240 $\pm$ 0.135 \\
 & CLoVE & 0.732 $\pm$ 0.092 & 1.289 $\pm$ 0.281 & 0.234 $\pm$ 0.068 & 0.231 $\pm$ 0.076 \\
 & VaREx & 0.721 $\pm$ 0.158 & 1.629 $\pm$ 0.849 & 0.247 $\pm$ 0.125 & 0.248 $\pm$ 0.137 \\
 & NGMM & \textbf{0.828 $\pm$ 0.000} & \textbf{0.508 $\pm$ 0.022} & \textbf{0.149 $\pm$ 0.003} & \textbf{0.093 $\pm$ 0.010} \\
\midrule
\multirow{5}{*}{PCBA-504332} & ERM & 0.649 $\pm$ 0.252 & 2.453 $\pm$ 2.514 & 0.318 $\pm$ 0.238 & 0.317 $\pm$ 0.246 \\
 & IRM & 0.653 $\pm$ 0.186 & 1.633 $\pm$ 1.437 & 0.290 $\pm$ 0.162 & 0.296 $\pm$ 0.172 \\
 & CLoVE & 0.728 $\pm$ 0.219 & 1.793 $\pm$ 1.931 & 0.245 $\pm$ 0.197 & 0.245 $\pm$ 0.206 \\
 & VaREx & 0.663 $\pm$ 0.231 & 2.051 $\pm$ 2.637 & 0.298 $\pm$ 0.223 & 0.298 $\pm$ 0.229 \\
 & NGMM & \textbf{0.866 $\pm$ 0.000} & \textbf{0.425 $\pm$ 0.022} & \textbf{0.120 $\pm$ 0.002} & \textbf{0.061 $\pm$ 0.011} \\
\midrule
\multirow{5}{*}{PCBA-686978} & ERM & 0.654 $\pm$ 0.204 & 2.436 $\pm$ 1.450 & 0.319 $\pm$ 0.188 & 0.314 $\pm$ 0.198 \\
 & IRM & 0.727 $\pm$ 0.030 & 1.874 $\pm$ 0.417 & 0.250 $\pm$ 0.022 & 0.243 $\pm$ 0.027 \\
 & CLoVE & 0.725 $\pm$ 0.045 & 1.490 $\pm$ 0.381 & 0.244 $\pm$ 0.025 & 0.229 $\pm$ 0.030 \\
 & VaREx & 0.693 $\pm$ 0.079 & 1.710 $\pm$ 0.178 & 0.273 $\pm$ 0.051 & 0.267 $\pm$ 0.054 \\
 & NGMM & \textbf{0.771 $\pm$ 0.002} & \textbf{0.619 $\pm$ 0.051} & \textbf{0.193 $\pm$ 0.004} & \textbf{0.129 $\pm$ 0.022} \\
\bottomrule
\end{tabular}%
}
\caption{Results for the OGB-MolPCBA scaffold generalization tasks. Across all four assays, NGMM achieves the highest accuracy and the lowest negative-log likelihood, Brier score, and expected calibration error. In all settings, NGMM also shows lowest variability across repetitions.} 
\label{tab:alpha_risk_specified}
\end{table}

\subsection{Camelyon-17}

The Camelyon17 dataset is a binary histopathology classification task. Each example is a pathology image patch, and the goal is to predict whether metastatic tumor tissue is present. Here the environments are hospitals. The goal is to learn predictors that generalize across hospital-specific variation in staining, scanning, and acquisition conditions.
For all methods we used the same small convolutional network. 

\paragraph{Results}

Here environments correspond to hospitals and no prior knowledge is available regarding whether the variation across hospitals affects the labeling mechanism.
Nevertheless, the results in Figure \ref{fig:camelyon} show that NGMM achieves the best overall performance on the held-out hospitals. It attains the highest accuracy and the lowest negative log-likelihood among all methods, indicating both stronger discrimination and better calibrated probabilistic predictions in unseen environments.

\begin{figure} [ht]
    \centering
    \includegraphics[width=0.8\linewidth]{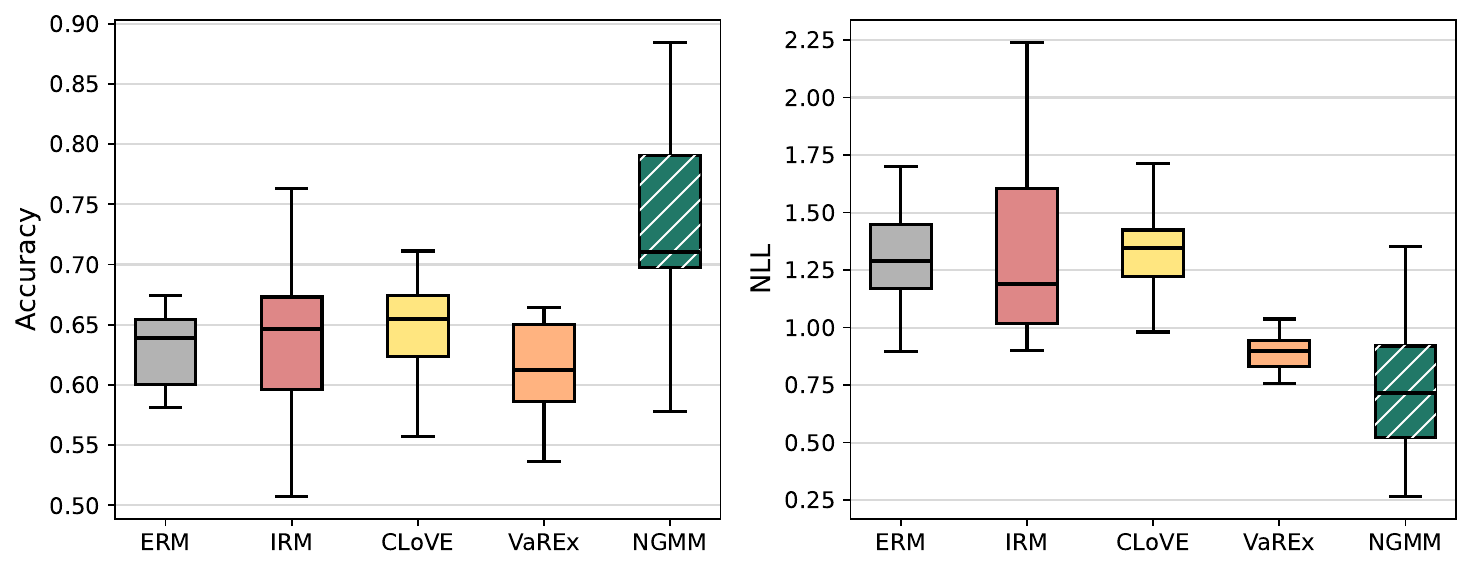}
    \caption{Results for the Camelyon-17 experiment on held-out hospitals. NGMM achieves the highest accuracy and the lowest negative log-likelihood among all methods, indicating the best predictive performance in unseen environments.
    }
    \label{fig:camelyon}
\end{figure}

\section{Comparison of robust and invariant representations} \label{sec:compare}

Having discussed robustness through marginalization of environment effects, we now turn to a comparison between this approach and invariance-based methods. In particular, we characterize when a robust representation, and especially one learned through a random-intercept model, is equivalent to an invariant one, and when it is expected to incur lower risk.

Invariant representation learning considers the probabilistic graphical model in Figure~\ref{fig:pgm_two_panel_a}, which corresponds to generative models of the form
\begin{equation} \label{eq:prob_model}
    \p(e, s, c, x, y) = \p(e)\, \p(s \mid e)\, \p(c \mid e)\, \p(x \mid s, c)\, \p(y \mid c).
\end{equation}

We assume that the model is non-trivial: both the causal features $c$ and the spurious features $s$ are non-degenerate latent variables (possibly aggregating several such variables correspondingly), and every solid arrow in Figure~\ref{fig:pgm_two_panel_a} represents a genuine dependence. That is, $x$ depends non-trivially on both $s$ and $c$, $y$ depends non-trivially on $c$, and $s$ depends non-trivially on $e$. We allow, though do not require, the causal factors $c$ to depend on $e$ as well.

\begin{definition}[Invariant representations]
Consider the generative model in Equation \eqref{eq:prob_model}, and a family $\mathcal{F}$ of representation functions $f: \mathcal{X} \to \mathcal{Z} \subseteq \R^d$.
We say that a representation $f \in \F$ is invariant if $y \perp e \mid f(x)$, and denote the subset of invariant representations as
\begin{equation}
    \Finva 
    \coloneqq
    \left\{
        f \in \F: 
        \;
        y \perp e \mid f(x)
    \right\}.
\end{equation}
\end{definition}

Invariant representations are defined through the conditional independence $y \perp e \mid f(x)$.
This captures the idea that, once the representation is given, the environment carries no further information about the target.
However, if an invariant representation exists, there may also be other representations that achieve the same predictive performance while retaining additional information that is unrelated to prediction, including information about the environment.

In contrast, if a robust representation is augmented with additional information that does not affect prediction, the resulting representation is still robust.
Hence, to compare invariance and robustness, we first introduce the following definition of \emph{minimality}.

\begin{definition}[Minimal representation] \label{def:minimal_robust}
A representation $f_{\varphi^\star}$ is called \emph{minimal} in $\F_0$, if for every $f_\varphi \in \F_0$, there exists a deterministic function $h_\varphi$ such that
\begin{equation}
    f_{\varphi^\star}(x)=h_\varphi(f_\varphi(x)).
\end{equation}
\end{definition}

Intuitively, the definition of minimality aligns with post-processing.
A minimal representation is one that can be recovered deterministically from any other representation in $\F_0$, and therefore does not retain unnecessary extra information.

The next proposition shows that, under the invariant data-generating model, once representations that carry such extra information are ruled out, a robust representation in $\F_0$ that is minimal, is also  invariant.

\begin{proposition} \label{prop:invariant_optimal}
Assume the model in Equation \eqref{eq:prob_model}, and let $\Phat$ be a family of conditional distributions indexed by $\theta=(\psi,\varphi)$.
Suppose that $\inf_{\theta\in\Theta} \risk_{\Phat}(\theta)=0$,
and that this infimum is attained.
Assume that there exists a representation $f_{\varphi^\dagger}\in\F_0$ that is invariant. Then, any robust representation $f_{\varphi^\star}\in\F_0$ that is minimal in $\F_0$ is also invariant.
\end{proposition}

The proof is provided in Appendix \ref{sup:proof_invariant_optimal}. The key idea of Proposition \ref{prop:invariant_optimal} is that 
when a representation already captures everything about $x$ that matters for prediction, any remaining dependence on the environment $e$ is redundant.
Therefore, a minimal robust representation has no need to retain any additional environment-specific information. 

However, the conditions in Proposition \ref{prop:invariant_optimal} in practice require full recovery of $c$ from $x$. While this may hold in some settings, it is not the case when $x$ is only a noisy or indirect proxy for $c$. In such cases, a robust (or $\varepsilon$-robust) representation may achieve lower risk by retaining information that is not invariant, but remains predictive.
In general, invariance and robustness do not coincide.

The distinction between robustness and invariance is especially clear for random-intercept models. Invariant methods posit that, after conditioning on a suitable representation, the effect of the environment disappears. Random-intercept models instead are designed for settings where residual environment variation remains and is modeled explicitly. Accordingly, they cannot recover an invariant representation exactly: if $c=f(x)$ and the conditional distribution includes an environment-specific random intercept $\gamma_e$, then  $\p(y \mid c,e)\neq \p(y \mid c)$. 

On the other hand, when the environment $e$ has a direct effect on $y$, the true data-generating process corresponds to the graphical model in Figure \ref{fig:pgm_two_panel_b}:
\begin{equation} \label{eq:prob_model2}
    \p(e, s, c, x, y) = \p(e)\, \p(s \mid e)\, \p(c \mid e)\, \p(x \mid s, c)\, \p(y \mid c, e).
\end{equation}
In this case the assumptions of invariant learning are violated, while the random-intercept model is recovered as the special case in which the effect of $e$ on $y$ enters through an additive intercept shift. 

In practice, the data-generating process may be closer to an invariant model or to a random-intercept model, and which approach performs better depends on which approximation is more appropriate. 

Let 
\begin{equation}
f_{\mathrm{ind}}
\in
\argmin_f I(y;e \mid f(x))
\end{equation}
be a representation targeting approximate invariance, and let $f_{\mathrm{RI}}$ be a population maximum-likelihood representation within the random-intercept family. 

Applying the the chain rule for conditional mutual information twice, in particular for $f\in \{f_{\mathrm{ind}}, f_{\mathrm{RI}}\}$, we have that
\begin{equation}
I(y;x \mid f(x))
=
I(y;e \mid f(x))
+
I(y;x \mid e,f(x))
-
I(y;e \mid x),
\end{equation}
and
\begin{equation}
I(y;x,e \mid f(x))
=
I(y;x \mid f(x))
+
I(y;e \mid x,f(x)).
\end{equation}
Since $f(x)$ is a deterministic function of $x$, conditioning on $f(x)$ adds no information once $x$ is given, and therefore
\begin{equation}
I(y;e \mid x,f(x)) = I(y;e \mid x).
\end{equation}
Combining the two expansions yields
\begin{equation}
I(y;x \mid f(x))
=
I(y;e \mid f(x))
+
I(y;x \mid e,f(x))
-
I(y;e \mid x).
\end{equation}

This shows that $f_{\mathrm{RI}}$ has smaller irreducible risk exactly when
\begin{equation} \label{eq:compare}
\delta
\coloneqq
\Bigl[
I(y;x \mid e,f_{\mathrm{ind}}(x))
-
I(y;x \mid e,f_{\mathrm{RI}}(x))
\Bigr]
-
\Bigl[
I(y;e \mid f_{\mathrm{RI}}(x))
-
I(y;e \mid f_{\mathrm{ind}}(x))
\Bigr] > 0.
\end{equation}
This inequality can be viewed as a tradeoff: forcing the representation to be as independent of $e$ as possible, may discard relevant predictive information in $x$ that remains useful once $e$ is known. When the effect of $e$ on $y$ is weak and most of the $y$-relevant information in $c$ can be recovered from $x$ alone, an invariance-seeking representation can be nearly optimal. By contrast, when $c$ cannot be well recovered from $x$, but $x$ contains additional environment-linked information that helps predict $y$, a random-intercept representation may be preferable. 

The empirical results in \S \ref{sec:simulation} evaluate exactly this tradeoff: When the direct effect of the environment is well captured by an additive intercept shift, the random-intercept model is correctly specified and therefore recovers the optimal marginal predictor. More importantly, when environment-linked information in $x$ helps predict $y$, modeling the residual environment effect can be preferable to trying to remove it.

This analysis is summarized in Figure \ref{fig:robust_invariant_summary}.

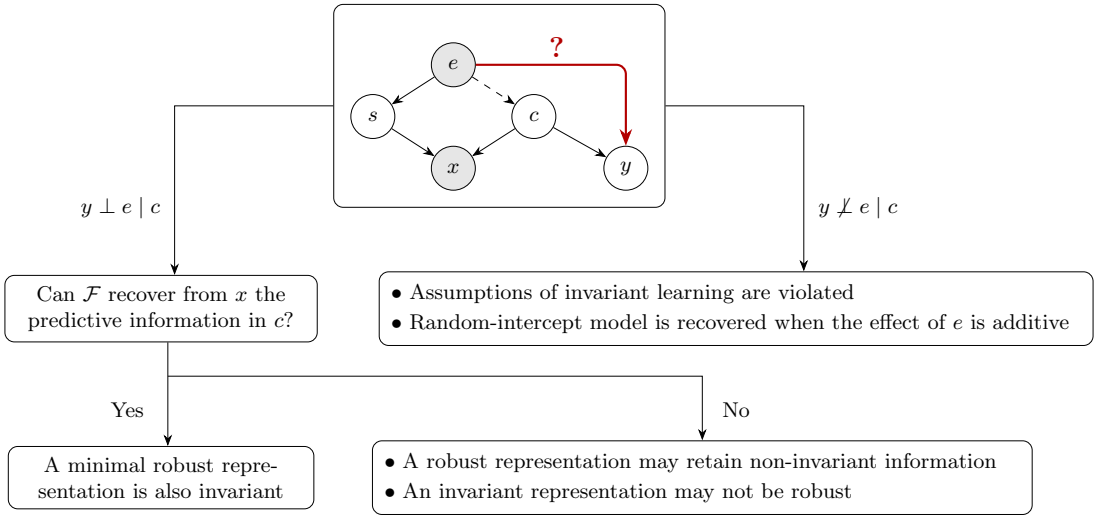
\begin{figure}[t]
\centering
\resizebox{0.95\linewidth}{!}{%
\begin{tikzpicture}[
    >=Stealth,
    font=\footnotesize,
    arrow/.style={thin,->},
    box/.style={
        draw,
        rounded corners,
        inner sep=4.5pt,
        align=center
    },
    bulletbox/.style={
        draw,
        rounded corners,
        inner sep=5pt,
        align=left
    },
    var/.style={circle, draw=black, thin, minimum size=6.5mm, inner sep=0pt, font=\footnotesize},
    grayvar/.style={var, fill=gray!20}
]

\node[draw, rounded corners, inner sep=7pt] (mini) at (0,0) {
\begin{tikzpicture}[>=Stealth, scale=0.85]
  \tikzset{
    var/.style={circle, draw=black, thin, minimum size=6.5mm, inner sep=0pt, font=\footnotesize},
    grayvar/.style={var, fill=gray!20}
  }

  \node[grayvar] (e) at (0,1.8) {$e$};
  \node[var]     (s) at (-1.4,0.9) {$s$};
  \node[var]     (c) at (1.4,0.9) {$c$};
  \node[grayvar] (x) at (0,0) {$x$};
  \node[var]     (y) at (3.0,0.0) {$y$};

  \draw[->] (e) -- (s);
  \draw[->] (s) -- (x);
  \draw[->] (c) -- (x);
  \draw[->] (c) -- (y);
  \draw[dashed, ->] (e) -- (c);
  \draw[->, red!70!black, line width=0.9pt] (e.east) -| (y.north);
  \node[text=red!70!black, font=\bfseries\large] at (1.8,2.1) {?};
\end{tikzpicture}
};

\node[box, text width=4.3cm] (question) at (-5.0,-3.0)
{Can $\F$ recover from $x$ the predictive information in $c$?};

\node[bulletbox, text width=10.2cm] (nocase) at (3.5,-3.0) {%
$\bullet$ Assumptions of invariant learning are violated\\[2pt]
$\bullet$ Random-intercept model is recovered when the effect of $e$ is additive
};

\draw[arrow] ($(mini.west)+(-0.0,0)$) -| ($(question.north)+(0.2,0)$);
\draw[arrow] ($(mini.east)+(0.0,0)$) -| ($(nocase.north)+(1,0)$);

\node at (-5.6,-1.50) {$y \perp e \mid c$};
\node at (5.3,-1.50) {$y \not \perp e \mid c$};

\node[box, text width=4.2cm] (invrob) at (-5,-5.5)
{A minimal robust representation is also invariant};

\node[bulletbox, text width=9.4cm] (tradeoff) at (+3,-5.5) {%
$\bullet$ A robust representation may retain non-invariant information\\[2pt]
$\bullet$ An invariant representation may not be robust
};

\draw[arrow] ($(question.south)+(0.1,0)$) -- ($(invrob.north)+(0.1,0)$);
\draw[arrow] ($(question.south)+(0.1,-0.5)$) -| ($(tradeoff.north)+(-0.0,0)$);

\node at (-5.5,-4.5) {Yes};
\node at (3.5,-4.5) {No};

\end{tikzpicture}%
}
\caption{Summary of the comparison between robustness and invariance.}
\label{fig:robust_invariant_summary}
\end{figure}

\section{Discussion} \label{sec:discuss}

We studied environment-robust representation learning in settings where the environment may directly affect the target. Such settings lie outside the standard invariance framework, which requires that, after conditioning on an appropriate representation, the predictive relationship no longer depends on the environment.

Invariance does not posit a structural model for how the environment affects the data-generating process. Accordingly, invariance-targeting methods seek a general solution: a representation under which these effects disappear. This is valuable when the goal is to identify stable underlying mechanisms. However, when the goal is prediction, our results show that explicitly modeling environment effects, rather than attempting to remove them, can lead to better performance. Neural random-intercept models provide a simple yet flexible way to do so.

We further showed that risk of such predictors decomposes into three terms: an irreducible term, a representation term $I(y;x \mid f(x))$, which measures the predictive information lost by compressing $x$ to $f(x)$, and an additional approximation term induced by restricting the predictor, conditional on that representation, to the marginalized family. This implies that when the random-intercept risk is small, the learned representation must retain substantial information in $x$ that is relevant for predicting $y$.

Our comparison with invariant representation learning analyzes when invariance is an appropriate target. Under the standard causal model, invariance and robustness coincide only when the representation preserves all posterior predictive information about $c$ that is relevant for $y$, and conditioning further on $e$ does not change that quantity. When this condition fails, enforcing invariance may remove predictive information that remains useful for robust prediction. This introduces a tradeoff between retaining a richer representation and increased environment dependence.
When the environment affects the target directly, invariance is unattainable altogether. In that case, robustness aiming methods such as the random-intercept model offer a natural alternative.

The empirical results are consistent with our analysis. In the    simulation, the gap between NGMM and invariance-targeting methods varies with the identified tradeoff. In all our experiments, the random-intercept model outperforms invariance-targeting methods. 

Nevertheless, the random-intercept family models environment effects through a single additive latent variable. This is appropriate when environments differ mainly through shifts of the shared predictor, but it does not capture other forms of heterogeneity.
Extensions to richer random-effect structures and to multimodal environment effects are a promising direction for future work. 

More broadly, our results suggest that in many settings robust representation learning should be formulated in terms of the prediction target and the class of shifts one aims to achieve robustness to, rather than through invariance alone.

\newpage

\bibliography{references}

\begin{thebibliography}{37}
\providecommand{\natexlab}[1]{#1}
\providecommand{\url}[1]{\texttt{#1}}
\expandafter\ifx\csname urlstyle\endcsname\relax
  \providecommand{\doi}[1]{doi: #1}\else
  \providecommand{\doi}{doi: \begingroup \urlstyle{rm}\Url}\fi

\bibitem[Ahuja et~al.(2020)Ahuja, Shanmugam, Varshney, and Dhurandhar]{ahuja2020invariant}
Kartik Ahuja, Karthikeyan Shanmugam, Kush Varshney, and Amit Dhurandhar.
\newblock Invariant risk minimization games.
\newblock In \emph{International Conference on Machine Learning}, pages 145--155. PMLR, 2020.

\bibitem[Arjovsky et~al.(2019)Arjovsky, Bottou, Gulrajani, and Lopez-Paz]{IRM}
Martin Arjovsky, L{\'e}on Bottou, Ishaan Gulrajani, and David Lopez-Paz.
\newblock Invariant risk minimization.
\newblock \emph{arXiv preprint arXiv:1907.02893}, 2019.

\bibitem[Baayen et~al.(2008)Baayen, Davidson, and Bates]{baayen2008mixed}
R~Harald Baayen, Douglas~J Davidson, and Douglas~M Bates.
\newblock Mixed-effects modeling with crossed random effects for subjects and items.
\newblock \emph{Journal of Memory and Language}, 59\penalty0 (4):\penalty0 390--412, 2008.

\bibitem[Ben-Tal et~al.(2013)Ben-Tal, Den~Hertog, De~Waegenaere, Melenberg, and Rennen]{DRO}
Aharon Ben-Tal, Dick Den~Hertog, Anja De~Waegenaere, Bertrand Melenberg, and Gijs Rennen.
\newblock Robust solutions of optimization problems affected by uncertain probabilities.
\newblock \emph{Management Science}, 59\penalty0 (2):\penalty0 341--357, 2013.

\bibitem[Bernal-Rusiel et~al.(2013)Bernal-Rusiel, Greve, Reuter, Fischl, Sabuncu, Initiative, et~al.]{bernal2013statistical}
Jorge~L Bernal-Rusiel, Douglas~N Greve, Martin Reuter, Bruce Fischl, Mert~R Sabuncu, Alzheimer's Disease~Neuroimaging Initiative, et~al.
\newblock Statistical analysis of longitudinal neuroimage data with linear mixed effects models.
\newblock \emph{NeuroImage}, 66:\penalty0 249--260, 2013.

\bibitem[Bryk and Raudenbush(1992)]{bryk1992hierarchical}
Anthony~S Bryk and Stephen~W Raudenbush.
\newblock \emph{Hierarchical Linear Models: Applications and Data Analysis Methods}.
\newblock Sage Publications, Inc, 1992.

\bibitem[Chang et~al.(2020)Chang, Zhang, Yu, and Jaakkola]{chang2020invariant}
Shiyu Chang, Yang Zhang, Mo~Yu, and Tommi Jaakkola.
\newblock Invariant rationalization.
\newblock In \emph{International Conference on Machine Learning}, pages 1448--1458. PMLR, 2020.

\bibitem[Duchi and Namkoong(2021)]{duchi2021learning}
John~C Duchi and Hongseok Namkoong.
\newblock Learning models with uniform performance via distributionally robust optimization.
\newblock \emph{The Annals of Statistics}, 49\penalty0 (3):\penalty0 1378--1406, 2021.

\bibitem[Duchi et~al.(2021)Duchi, Glynn, and Namkoong]{duchi2021statistics}
John~C Duchi, Peter~W Glynn, and Hongseok Namkoong.
\newblock Statistics of robust optimization: A generalized empirical likelihood approach.
\newblock \emph{Mathematics of Operations Research}, 46\penalty0 (3):\penalty0 946--969, 2021.

\bibitem[Ganin et~al.(2016)Ganin, Ustinova, Ajakan, Germain, Larochelle, Laviolette, March, and Lempitsky]{ganin2016domain}
Yaroslav Ganin, Evgeniya Ustinova, Hana Ajakan, Pascal Germain, Hugo Larochelle, Fran{\c{c}}ois Laviolette, Mario March, and Victor Lempitsky.
\newblock Domain-adversarial training of neural networks.
\newblock \emph{Journal of Machine Learning Research}, 17\penalty0 (59):\penalty0 1--35, 2016.

\bibitem[Gelman et~al.(2007)Gelman, Fagan, and Kiss]{gelman2007analysis}
Andrew Gelman, Jeffrey Fagan, and Alex Kiss.
\newblock An analysis of the new york city police department's “stop-and-frisk” policy in the context of claims of racial bias.
\newblock \emph{Journal of the American Statistical Association}, 102\penalty0 (479):\penalty0 813--823, 2007.

\bibitem[Idrissi et~al.(2022)Idrissi, Arjovsky, Pezeshki, and Lopez-Paz]{idrissi2022simple}
Badr~Youbi Idrissi, Martin Arjovsky, Mohammad Pezeshki, and David Lopez-Paz.
\newblock Simple data balancing achieves competitive worst-group-accuracy.
\newblock In \emph{Conference on Causal Learning and Reasoning}, pages 336--351. PMLR, 2022.

\bibitem[Judd et~al.(2012)Judd, Westfall, and Kenny]{judd2012treating}
Charles~M Judd, Jacob Westfall, and David~A Kenny.
\newblock Treating stimuli as a random factor in social psychology: A new and comprehensive solution to a pervasive but largely ignored problem.
\newblock \emph{Journal of Personality and Social Psychology}, 103\penalty0 (1):\penalty0 54, 2012.

\bibitem[Krueger et~al.(2021)Krueger, Caballero, Jacobsen, Zhang, Binas, Zhang, Le~Priol, and Courville]{krueger2021out}
David Krueger, Ethan Caballero, Joern-Henrik Jacobsen, Amy Zhang, Jonathan Binas, Dinghuai Zhang, Remi Le~Priol, and Aaron Courville.
\newblock Out-of-distribution generalization via risk extrapolation ({REx}).
\newblock In \emph{International Conference on Machine Learning}, pages 5815--5826. PMLR, 2021.

\bibitem[Lin et~al.(2022)Lin, Dong, Wang, and Zhang]{lin2022bayesian}
Yong Lin, Hanze Dong, Hao Wang, and Tong Zhang.
\newblock Bayesian invariant risk minimization.
\newblock In \emph{Proceedings of the IEEE/CVF Conference on Computer Vision and Pattern Recognition}, pages 16021--16030, 2022.

\bibitem[Listgarten et~al.(2010)Listgarten, Kadie, Schadt, and Heckerman]{listgarten2010correction}
Jennifer Listgarten, Carl Kadie, Eric~E Schadt, and David Heckerman.
\newblock Correction for hidden confounders in the genetic analysis of gene expression.
\newblock \emph{Proceedings of the National Academy of Sciences}, 107\penalty0 (38):\penalty0 16465--16470, 2010.

\bibitem[Liu et~al.(2021)Liu, Haghgoo, Chen, Raghunathan, Koh, Sagawa, Liang, and Finn]{liu2021just}
Evan~Z Liu, Behzad Haghgoo, Annie~S Chen, Aditi Raghunathan, Pang~Wei Koh, Shiori Sagawa, Percy Liang, and Chelsea Finn.
\newblock Just train twice: Improving group robustness without training group information.
\newblock In \emph{International Conference on Machine Learning}, pages 6781--6792. PMLR, 2021.

\bibitem[Lu et~al.(2021)Lu, Wu, Hern{\'a}ndez-Lobato, and Sch{\"o}lkopf]{lu2021invariant}
Chaochao Lu, Yuhuai Wu, Jos{\'e}~Miguel Hern{\'a}ndez-Lobato, and Bernhard Sch{\"o}lkopf.
\newblock Invariant causal representation learning for out-of-distribution generalization.
\newblock In \emph{International Conference on Learning Representations}, 2021.

\bibitem[Lyu et~al.(2023)Lyu, Kim, and Suk]{lyu2023estimating}
Weicong Lyu, Jee-Seon Kim, and Youmi Suk.
\newblock Estimating heterogeneous treatment effects within latent class multilevel models: A bayesian approach.
\newblock \emph{Journal of Educational and Behavioral Statistics}, 48\penalty0 (1):\penalty0 3--36, 2023.

\bibitem[Mahajan et~al.(2021)Mahajan, Tople, and Sharma]{mahajan2021domain}
Divyat Mahajan, Shruti Tople, and Amit Sharma.
\newblock Domain generalization using causal matching.
\newblock In \emph{International Conference on Machine Learning}, pages 7313--7324. PMLR, 2021.

\bibitem[Nye et~al.(2000)Nye, Hedges, and Konstantopoulos]{nye2000effects}
Barbara Nye, Larry~V Hedges, and Spyros Konstantopoulos.
\newblock The effects of small classes on academic achievement: The results of the tennessee class size experiment.
\newblock \emph{American Educational Research Journal}, 37\penalty0 (1):\penalty0 123--151, 2000.

\bibitem[Payne et~al.(2015)Payne, Lee, and Federmeier]{payne2015revisiting}
Brennan~R Payne, Chia-Lin Lee, and Kara~D Federmeier.
\newblock Revisiting the incremental effects of context on word processing: Evidence from single-word event-related brain potentials.
\newblock \emph{Psychophysiology}, 52\penalty0 (11):\penalty0 1456--1469, 2015.

\bibitem[Peters et~al.(2016)Peters, B{\"u}hlmann, and Meinshausen]{peters2016causal}
Jonas Peters, Peter B{\"u}hlmann, and Nicolai Meinshausen.
\newblock Causal inference by using invariant prediction: Identification and confidence intervals.
\newblock \emph{Journal of the Royal Statistical Society Series B: Statistical Methodology}, 78\penalty0 (5):\penalty0 947--1012, 2016.

\bibitem[Peters et~al.(2017)Peters, Janzing, and Sch{\"o}lkopf]{peters2017elements}
Jonas Peters, Dominik Janzing, and Bernhard Sch{\"o}lkopf.
\newblock \emph{Elements of Causal Inference: Foundations and Learning Algorithms}.
\newblock The MIT Press, 2017.

\bibitem[Pfister et~al.(2019)Pfister, B{\"u}hlmann, and Peters]{pfister2019invariant}
Niklas Pfister, Peter B{\"u}hlmann, and Jonas Peters.
\newblock Invariant causal prediction for sequential data.
\newblock \emph{Journal of the American Statistical Association}, 114\penalty0 (527):\penalty0 1264--1276, 2019.

\bibitem[Piratla et~al.(2021)Piratla, Netrapalli, and Sarawagi]{piratla2021focus}
Vihari Piratla, Praneeth Netrapalli, and Sunita Sarawagi.
\newblock Focus on the common good: Group distributional robustness follows.
\newblock In \emph{International Conference on Learning Representations}, 2021.

\bibitem[Raudenbush et~al.(1999)Raudenbush, Fotiu, and Cheong]{raudenbush1999synthesizing}
Stephen~W Raudenbush, Randall~P Fotiu, and Yuk~Fai Cheong.
\newblock Synthesizing results from the trial state assessment.
\newblock \emph{Journal of Educational and Behavioral Statistics}, 24\penalty0 (4):\penalty0 413--438, 1999.

\bibitem[Rojas-Carulla et~al.(2018)Rojas-Carulla, Sch{\"o}lkopf, Turner, and Peters]{rojas2018invariant}
Mateo Rojas-Carulla, Bernhard Sch{\"o}lkopf, Richard Turner, and Jonas Peters.
\newblock Invariant models for causal transfer learning.
\newblock \emph{Journal of Machine Learning Research}, 19\penalty0 (36):\penalty0 1--34, 2018.

\bibitem[Rothenh{\"a}usler et~al.(2021)Rothenh{\"a}usler, Meinshausen, B{\"u}hlmann, and Peters]{rothenhausler2021anchor}
Dominik Rothenh{\"a}usler, Nicolai Meinshausen, Peter B{\"u}hlmann, and Jonas Peters.
\newblock Anchor regression: Heterogeneous data meet causality.
\newblock \emph{Journal of the Royal Statistical Society Series B: Statistical Methodology}, 83\penalty0 (2):\penalty0 215--246, 2021.

\bibitem[Sagawa et~al.(2019)Sagawa, Koh, Hashimoto, and Liang]{sagawa2019distributionally}
Shiori Sagawa, Pang~Wei Koh, Tatsunori~B Hashimoto, and Percy Liang.
\newblock Distributionally robust neural networks.
\newblock In \emph{International Conference on Learning Representations}, 2019.

\bibitem[Sinha et~al.(2017)Sinha, Namkoong, Volpi, and Duchi]{sinha2017certifying}
Aman Sinha, Hongseok Namkoong, Riccardo Volpi, and John Duchi.
\newblock Certifying some distributional robustness with principled adversarial training.
\newblock \emph{arXiv preprint arXiv:1710.10571}, 2017.

\bibitem[Slavutsky et~al.(2026)Slavutsky, Salazar, and Blei]{slavutsky2026neural}
Yuli Slavutsky, Sebastian Salazar, and David Blei.
\newblock Neural generalized mixed-effects models.
\newblock \emph{arXiv preprint arXiv:2604.10976}, 2026.

\bibitem[Visscher et~al.(2003)Visscher, Miezin, Kelly, Buckner, Donaldson, McAvoy, Bhalodia, and Petersen]{visscher2003mixed}
Kristina~M Visscher, Francis~M Miezin, James~E Kelly, Randy~L Buckner, David~I Donaldson, Mark~P McAvoy, Vidya~M Bhalodia, and Steven~E Petersen.
\newblock Mixed blocked/event-related designs separate transient and sustained activity in {fMRI}.
\newblock \emph{NeuroImage}, 19\penalty0 (4):\penalty0 1694--1708, 2003.

\bibitem[Wald et~al.(2021)Wald, Feder, Greenfeld, and Shalit]{wald2021calibration}
Yoav Wald, Amir Feder, Daniel Greenfeld, and Uri Shalit.
\newblock On calibration and out-of-domain generalization.
\newblock \emph{Advances in Neural Information Processing Systems}, 34:\penalty0 2215--2227, 2021.

\bibitem[Wei et~al.(2023)Wei, Narasimhan, Amid, Chu, Liu, and Kumar]{wei2023distributionally}
Jiaheng Wei, Harikrishna Narasimhan, Ehsan Amid, Wen-Sheng Chu, Yang Liu, and Abhishek Kumar.
\newblock Distributionally robust post-hoc classifiers under prior shifts.
\newblock In \emph{International Conference on Learning Representations}, 2023.

\bibitem[Zhang et~al.(2010)Zhang, Ersoz, Lai, Todhunter, Tiwari, Gore, Bradbury, Yu, Arnett, Ordovas, et~al.]{zhang2010mixed}
Zhiwu Zhang, Elhan Ersoz, Chao-Qiang Lai, Rory~J Todhunter, Hemant~K Tiwari, Michael~A Gore, Peter~J Bradbury, Jianming Yu, Donna~K Arnett, Jose~M Ordovas, et~al.
\newblock Mixed linear model approach adapted for genome-wide association studies.
\newblock \emph{Nature Genetics}, 42\penalty0 (4):\penalty0 355--360, 2010.

\bibitem[Zhou and Stephens(2012)]{zhou2012genome}
Xiang Zhou and Matthew Stephens.
\newblock Genome-wide efficient mixed-model analysis for association studies.
\newblock \emph{Nature Genetics}, 44\penalty0 (7):\penalty0 821--824, 2012.

\end{thebibliography}

\newpage

\appendix

\section{Proofs}
\subsection{Proof of Lemma \ref{lemma:same_min}} \label{sup:proof_same_min}

\begin{proof}

For a fixed $\theta \in \Theta$, the environment-average risk is
\begin{align}
     \bar{\risk}_{\Phat}(\theta) & \coloneqq 
    \E_{\p(e,x)} \left[ 
    \KL\left( \p(y \mid x, e) \, \Vert \, \phat_\theta(y \mid x) \right)
    \right] \\
    & = \E_{\p(e,x)}\left[
    \E_{\p(y \mid e,x)}
    \left[ \log \frac{\p(y \mid x,e)}{\phat_\theta(y \mid x)}\right]
    \right] \\
    & = \E_{\p(e,x, y)}\left[
    \log \frac{\p(y \mid x,e)}{\phat_\theta(y \mid x)}
    \right].
\end{align}
By adding and subtracting $\log \p(y \mid x)$, we get
\begin{equation}
\bar{\risk}_{\Phat}(\theta)
=
\E_{\p(e,x,y)}
\left[
\log \frac{\p(y \mid x,e)}{\p(y \mid x)}
\right]
+
\E_{\p(e,x,y)}
\left[
\log \frac{\p(y \mid x)}{\phat_\theta(y \mid x)}
\right].
\end{equation}

Note that the first term can be expressed as
\begin{equation}
\E_{\p(e,x,y)}
\left[
\log \frac{\p(y \mid x,e)}{\p(y \mid x)}
\right]
=
\E_{\p(e,x)}
\left[
\KL\left(
\p(y \mid x,e)
\Vert
\p(y \mid x)
\right)
\right],
\end{equation}
and it does not depend on $\theta$.

For the second term, we have that 
\begin{equation}
\E_{\p(e,x,y)}
\left[
\log \frac{\p(y \mid x)}{\phat_\theta(y \mid x)}
\right]
=
\E_{\p(x,y)}
\left[
\log \frac{\p(y \mid x)}{\phat_\theta(y \mid x)}
\right]
=
\risk_{\Phat}(\theta).
\end{equation}
Therefore,
\begin{equation}
\bar{\risk}_{\Phat}(\theta)
=
\risk_{\Phat}(\theta)
+
\E_{\p(e,x)}
\left[
\KL\left(
\p(y \mid x,e)
\, \Vert \, 
\p(y \mid x)
\right)
\right].
\end{equation}
\end{proof}

\subsection{Proof of Proposition \ref{prop:decomposition}} \label{sup:proof_decomposition}

\begin{proof}
By definition,
\begin{align}
\risk_{\Phat}(\theta)
& =
\E_{\p(x)}
\left[
\KL\left(
\p(y \mid x)
\Vert
\phat_{\theta}\left(y \mid x\right)
\right)
\right] \\
& =
\E_{\p(x)\p(y\mid x)}
\left[
\log
\frac{
\p(y \mid x)
}{
\phat_{\theta}\left(y \mid x\right)
}
\right].
\end{align}
By adding and subtracting $\log \p\left(y \mid f_\varphi(x)\right)$, we get
\begin{equation} \label{eq:decomp}
\begin{aligned}
\risk_{\Phat}(\theta)
&=
\E_{\p(x)\p(y\mid x)}
\left[
\log
\frac{
\p(y \mid x)
}{
\p\left(y \mid f_\varphi(x)\right)
}
\right]
+
\E_{\p(x)\p(y\mid x)}
\left[
\log
\frac{
\p\left(y \mid f_\varphi(x)\right)
}{
\phat_{\theta}\left(y \mid x\right)
}
\right],
\end{aligned}
\end{equation}
where the first term can be rewritten as
\begin{align}
\E_{\p(x)\p(y\mid x)}
\left[
\log
\frac{
\p(y \mid x)
}{
\p\left(y \mid f_\varphi(x)\right)
}
\right]
&=
\E_{\p(x)}
\left[
\KL\left(
\p(y \mid x)
\, \Vert \,
\p\left(y \mid f_\varphi(x)\right)
\right)
\right].
\end{align}
Below for brevity we drop the index $\varphi$ and write $f$ instead of $f_\varphi$.

We will first show that
\begin{equation}
\E_{\p(x)}
\left[
\KL\left(
\p(y \mid x)
\, \Vert \,
\p\left(y \mid f(x)\right)
\right)
\right]
=
I\left(y;x \mid f(x)\right).
\end{equation}
Note that
\begin{align}
\p(x,y \mid f(x))
&=
\p(x \mid f(x)) \, \p(y \mid x, f(x)) \\
&=
\p(x \mid f(x)) \, \p(y \mid x),
\end{align}
where the second equality follows from $f$ being a deterministic function.
Then,
\begin{align}
I\left(y;x \mid f(x)\right)
&\coloneqq
\E_{\p(f(x))}
\left[
\KL\left(
\p(x,y \mid f(x))
\, \Vert \,
\p(x \mid f(x))\, \p(y \mid f(x))
\right)
\right] \\
&=
\E_{\p(f(x))}
\left[
\KL\left(
\p(x \mid f(x))\, \p(y \mid x)
\, \Vert \,
\p(x \mid f(x))\, \p(y \mid f(x))
\right)
\right],
\end{align}
where
\begin{align}
& \KL\left(\p(x\mid f(x))\,\p(y\mid x)\Vert\p(x\mid f(x))\,\p(y\mid f(x))\right)\\
& =
\E_{\p(x\mid f(x))}
\left[
\E_{\p(y\mid x)}
\left[
\log
\frac{
\p(x\mid f(x))\,\p(y\mid x)
}{
\p(x\mid f(x))\,\p(y\mid f(x))
}
\right]
\right]
\\
& =
\E_{\p(x\mid f(x))}
\left[
\KL\left(
\p(y\mid x)
\, \Vert \,
\p(y\mid f(x))
\right)
\right].
\end{align}
Therefore,
\begin{align}
I\left(y;x \mid f(x)\right)
&=
\E_{\p(f(x))}
\left[
\E_{\p(x \mid f(x))}
\left[
\KL\left(
\p(y \mid x)
\, \Vert \,
\p(y \mid f(x))
\right)
\right]
\right],
\end{align}
and since $\p(x, f(x))=\p(x)$ for any deterministic $f$,
\begin{align}
I\left(y;x \mid f(x)\right)
&=
\E_{\p(x)}
\left[
\KL\left(
\p(y \mid x)
\, \Vert \,
\p(y \mid f(x))
\right)
\right].
\end{align}

We now move to the second term in Equation \eqref{eq:decomp}.
Recall that $\phat_\theta(y \mid x) = \phat_\psi(y \mid f_\varphi(x))$, and denote 

\begin{equation}
Z \coloneqq f_\varphi(X), \qquad
g(y,z) \coloneqq \log \frac{\p(y \mid z)}{\phat_\psi(y \mid z)}.
\end{equation}
Then,
\begin{align}
\E_{\p(x)\p(y\mid x)}
\left[
\log
\frac{
\p\left(y \mid f_\varphi(x)\right)
}{
\phat_{\psi}\left(y \mid x\right)
}
\right]
& =
\E_{\p(z)}
\left[
\E_{\p(x,y \mid z)}
\left[
g(y,z)
\right]
\right]
\\
& =
\E_{\p(z)}
\left[
\E_{\p(y \mid z)}
\left[
g(y,z)
\right]
\right]
\\
&=
\E_{\p(z)}
\left[
\KL\left(
\p(y \mid z)
\, \Vert \,
\phat_\psi(y \mid z)
\right)
\right].
\end{align}
Substituting back $z=f_\varphi(x)$ yields
\begin{align}
\E_{\p(x)\p(y\mid x)}
\left[
\log
\frac{
\p(y \mid f_\varphi(x))
}{
\phat_\theta(y \mid x)
}
\right]
& =
\E_{p(f_\varphi(x))}
\left[
\KL\left(
\p(y \mid f_\varphi(x))
\, \Vert \,
\phat_\psi(y \mid f_\varphi(x))
\right)
\right]\\
& =
\E_{\p(x)}
\left[
\KL\left(
\p(y \mid f_\varphi(x))
\, \Vert \,
\phat_\theta(y \mid x)
\right)
\right] = \Delta_{\Phat \mid f}.
\end{align}
\end{proof}

\subsection{Proof of Proposition \ref{prop:invariant_optimal}}

\begin{proof} \label{sup:proof_invariant_optimal}
Let $\Phat$ be a family of conditional distributions indexed by $\theta=(\psi,\varphi)$, where $\varphi$ parameterizes representations in $\F$ and $\psi$ denotes the remaining model parameters. 

Assume that the model is well-specified so that
\begin{equation}
    \inf_{\theta\in\Theta}\risk_{\Phat}(\theta)=0,
\end{equation}
and that this infimum is attained. Let $f_{\varphi^\star}$ be a robust representation and assume that $f_{\varphi^\star}$ is minimal in $\F_0$. Then, by robustness, there exists $\psi^\star$ such that for $\theta^\star=(\psi^\star,\varphi^\star)$,
\begin{equation}
    \theta^\star \in \argmin_{\theta\in\Theta}\bar{\risk}_{\Phat}(\theta).
\end{equation}

By Lemma \ref{lemma:same_min}, the minimizers of $\bar{\risk}_{\Phat}$ and $\risk_{\Phat}$ coincide. Therefore,
\begin{equation}
    \risk_{\Phat}(\theta^\star)
    =
    \E_{\p(x)}
    \left[
        \KL\left(
            \p(y \mid x)
            \Vert
            \phat_{\psi^\star}(y \mid f_{\varphi^\star}(x))
        \right)
    \right]
    =0.
\end{equation}

Adding and subtracting $\log \p(y \mid f_{\varphi^\star}(x))$ gives
\begin{align}
    \risk_{\Phat}(\theta^\star)
    &=
    \E_{\p(x)} \E_{\p(y \mid x)}
    \left[
        \log \frac{\p(y \mid x)}{\p(y \mid f_{\varphi^\star}(x))}
        +
        \log \frac{\p(y \mid f_{\varphi^\star}(x))}{\phat_{\psi^\star}(y \mid f_{\varphi^\star}(x))}
    \right] \\
    &=
    I(y;x \mid f_{\varphi^\star}(x))
    +
    \E_{p(x)}
    \left[
        \KL\left(
            p(y \mid f_{\varphi^\star}(x))
            \,\Vert\,
            \phat_{\psi^\star}(y \mid f_{\varphi^\star}(x))
        \right)
    \right]
    =0. \label{eq:last}
\end{align}
Since both terms are nonnegative it follows that
\begin{equation}
    I(y;x \mid f_{\varphi^\star}(x))=0.
\end{equation}

Now, under the model in Equation \eqref{eq:prob_model}, for any deterministic representation $f$,
\begin{align}
    & \p(y\mid x) = \E_{\p(c\mid x)}[\p(y \mid c)], \\
    & \p(y\mid f(x)) = \E_{\p(c\mid f(x))}[\p(y \mid c)].
\end{align}
Hence
\begin{equation}
    I(y;x \mid f(x))
    =
    \E_{\p(x)}
    \left[
        \KL\left(
            \E_{\p(c\mid x)}[\p(y \mid c)]
            \,\Vert\,
            \E_{\p(c\mid f(x))}[\p(y \mid c)]
        \right)
    \right].
\end{equation}
Applying this with $f=f_{\varphi^\star}$ yields
\begin{equation} \label{eq:1}
    \E_{\p(c\mid f_{\varphi^\star}(x))}[\p(y \mid c)]
    =
    \E_{\p(c\mid x)}[\p(y \mid c)]
    \qquad
    \text{almost surely in }\p(x).
\end{equation}

By assumption, there exists a robust invariant representation $f_{\varphi^\dagger}$. Since $f_{\varphi^\star}$ is minimal in $\F_0$ and $f_{\varphi^\dagger}\in\F_0$, there exists a function $h$ such that
\begin{equation}
    f_{\varphi^\star}(x)=h(f_{\varphi^\dagger}(x))
    \qquad
    \text{almost surely in }\p(x).
\end{equation}
By invariance of $f_{\varphi^\dagger}$,
\begin{equation}
    \p(y \mid f_{\varphi^\dagger}(x),e)=\p(y \mid f_{\varphi^\dagger}(x)).
\end{equation}
Since $f_{\varphi^\star}(x)$ is a function of $f_{\varphi^\dagger}(x)$, we have
\begin{align}
    \p(y \mid f_{\varphi^\star}(x),e)
    &=
    \E\left[\p(y \mid f_{\varphi^\dagger}(x),e)\mid f_{\varphi^\star}(x),e\right] \\
    &=
    \E\left[\p(y \mid f_{\varphi^\dagger}(x))\mid f_{\varphi^\star}(x),e\right].
\end{align}
By robustness of $f_{\varphi^\dagger}$ and Equation \eqref{eq:1},
\begin{equation}
    \p(y \mid f_{\varphi^\dagger}(x))
    =
    \p(y \mid x)
    =
    \p(y \mid f_{\varphi^\star}(x))
    \qquad
    \text{almost surely in }\p(x), 
\end{equation}
and therefore,
\begin{align}
    \p(y \mid f_{\varphi^\star}(x),e)
    &=
    \E\left[\p(y \mid f_{\varphi^\star}(x))\mid f_{\varphi^\star}(x),e\right] \\
    &=
    \p(y \mid f_{\varphi^\star}(x)).
\end{align}
\end{proof}

\section{Risk derivations for the tradeoff simulation} 
Let
\begin{align}
 & e_{j}\sim\mathcal{N}(0,\sigma_{e}^{2}),\\
 & c_{ij}\mid e_{j}\sim\mathcal{N}(e_{j},\sigma_{c}^{2}),\\
 & s_{ij}=\alpha e_{j}+u_{ij},\\
 & u_{ij}\sim\mathcal{N}(0,\sigma_{u}^{2}),\\
 & \varepsilon_{ij}\sim\mathcal{N}(0,\sigma_{y}^{2}).
\end{align}

\subsection{Well-specified model} \label{sec:spec_deriv}

We first consider the model 
\begin{equation}
y_{ij} = c_{ij} + \alpha e_j + \varepsilon_{ij},
\end{equation}
with $x_{ij}=(c_{ij},s_{ij})$. 
Below we derive 
\begin{equation}
\bar{\risk}^\star
=
\E_{\p(e)\p(x \mid e)}
\left[
\KL\left( \p(y \mid x,e) \,\Vert\, \p(y \mid x) \right)
\right].
\end{equation}
Suppressing index notation, since
\begin{equation}
y = c + \alpha e + \varepsilon,
\end{equation}
conditional on $(c,s,e)$ we have
\begin{equation}
\p(y \mid c,s,e)
=
\mathcal N(c+\alpha e,\sigma_{\epsilon}^2).
\end{equation}

Next, we have
\begin{equation}
\E[y \mid c,s]
=
c + \alpha \E[e \mid c,s],
\end{equation}
and
\begin{equation}
\Var(y \mid c,s)
=
\alpha^2 \Var(e \mid c,s) + \sigma_{\epsilon}^2,
\end{equation}
so
\begin{equation}
\p(y \mid c,s)
=
\mathcal N\left(
c+\alpha \E[e \mid c,s],
\sigma_{\epsilon}^2+\alpha^2 \Var(e \mid c,s)
\right).
\end{equation}

Since $(e,c,s)$ is jointly Gaussian, 
\begin{equation}
V_\alpha \coloneqq \Var(e \mid c,s)
=
\left(
\frac{1}{\sigma_e^2}
+
\frac{1}{\sigma_c^2}
+
\frac{\alpha^2}{\sigma_u^2}
\right)^{-1}.
\end{equation}

Therefore, we have
\begin{align} \notag
& \KL\left( \p(y \mid c,s,e) \,\Vert\, \p(y \mid c,s) \right) \\
  & =
\KL\left(
\mathcal N(c+\alpha e,\sigma_{\epsilon}^2)
\;\Vert\;
\mathcal N\left(c+\alpha \E[e \mid c,s],\sigma_{\epsilon}^2+\alpha^2 V_\alpha\right)
\right).
\end{align}

Using closed form KL for Gaussian distributions,
\begin{equation}
\KL\left( \mathcal N(\mu_1,v_1) \,\Vert\, \mathcal N(\mu_2,v_2) \right)
=
\frac{1}{2}
\left[
\log\frac{v_2}{v_1}
+
\frac{v_1+(\mu_1-\mu_2)^2}{v_2}
-1
\right],
\end{equation}
we obtain
\begin{equation}
\KL\left( \p(y \mid c,s,e) \,\Vert\, \p(y \mid c,s) \right)
=
\frac{1}{2}
\left[
\log\frac{\sigma_{\epsilon}^2+\alpha^2 V_\alpha}{\sigma_{\epsilon}^2}
+
\frac{
\sigma_{\epsilon}^2+\alpha^2 (e-\E[e \mid c,s])^2
}{
\sigma_{\epsilon}^2+\alpha^2 V_\alpha
}
-1
\right].
\end{equation}
Taking expectation with respect to $(e,c,s)$ yields
\begin{equation}
\bar{\risk}^\star
=
\frac{1}{2}
\left[
\log\frac{\sigma_{\epsilon}^2+\alpha^2 V_\alpha}{\sigma_{\epsilon}^2}
+
\frac{
\sigma_{\epsilon}^2+\alpha^2 \E\left[(e-\E[e \mid c,s])^2\right]
}{
\sigma_{\epsilon}^2+\alpha^2 V_\alpha
}
-1
\right].
\end{equation}

Finally,
\begin{equation}
\E\left[(e-\E[e \mid c,s])^2\right]
=
\E\left[\Var(e \mid c,s)\right]
=
V_\alpha,
\end{equation}
so the second term simplifies, and we get that
\begin{equation}
\bar{\risk}^\star
=
\frac{1}{2}
\log\left(
1+\frac{\alpha^2 V_\alpha}{\sigma_{\epsilon}^2}
\right).
\end{equation}

\subsection{Misspecified model} \label{sec:misspec_deriv}

We now consider the model
\begin{equation}
y_{ij}
=
c_{ij}
+
\alpha e_j
+
(1-\alpha)e_j c_{ij}
+
\varepsilon_{ij},
\end{equation}
with $x_{ij}=(c_{ij},s_{ij})$.
Suppressing index notation again, since
\begin{equation}
y
=
c
+
\bigl(\alpha + (1-\alpha)c\bigr)e
+
\varepsilon,
\end{equation}
conditional on $(c,s,e)$ we have
\begin{equation}
\p(y \mid c,s,e)
=
\mathcal N\bigl(c+\bigl(\alpha + (1-\alpha)c\bigr)e,\sigma_{\epsilon}^2\bigr).
\end{equation}

Next, $(e,c,s)$ is jointly Gaussian and therefore
\begin{equation}
e \mid c,s \sim \mathcal N(m(c,s),V_\alpha),
\end{equation}
where
\begin{equation}
m(c,s)
=
V_\alpha
\left(
\frac{c}{\sigma_c^2}
+
\frac{\alpha s}{\sigma_u^2}
\right),
\qquad
V_\alpha
=
\left(
\frac{1}{\sigma_e^2}
+
\frac{1}{\sigma_c^2}
+
\frac{\alpha^2}{\sigma_u^2}
\right)^{-1}.
\end{equation}

This yields
\begin{equation}
\E[y \mid c,s]
=
c+\bigl(\alpha + (1-\alpha)c\bigr)\E[e \mid c,s]
=
c+\bigl(\alpha + (1-\alpha)c\bigr)m(c,s),
\end{equation}
and
\begin{equation}
\Var(y \mid c,s)
=
\bigl(\alpha + (1-\alpha)c\bigr)^2 \Var(e \mid c,s)+\sigma_{\epsilon}^2
=
\bigl(\alpha + (1-\alpha)c\bigr)^2 V_\alpha+\sigma_{\epsilon}^2,
\end{equation}
so
\begin{equation}
\p(y \mid c,s)
=
\mathcal N\left(
c+\bigl(\alpha + (1-\alpha)c\bigr)m(c,s),
\sigma_{\epsilon}^2+\bigl(\alpha + (1-\alpha)c\bigr)^2 V_\alpha
\right).
\end{equation}

Using the closed-form expression for the KL divergence between Gaussian distributions again, we get
\begin{align} \notag
& \KL\left( \p(y \mid c,s,e) \,\Vert\, \p(y \mid c,s) \right) \\
&=
\frac{1}{2}
\left[
\log\frac{\sigma_{\epsilon}^2+\bigl(\alpha + (1-\alpha)c\bigr)^2 V_\alpha}{\sigma_{\epsilon}^2}
+
\frac{
\sigma_{\epsilon}^2+\bigl(\alpha + (1-\alpha)c\bigr)^2 (e-\E[e \mid c,s])^2
}{
\sigma_{\epsilon}^2+\bigl(\alpha + (1-\alpha)c\bigr)^2 V_\alpha
}
-1
\right].
\end{align}
Taking expectation with respect to $(e,c,s)$ yields
\begin{align} \notag
\bar{\risk}^\star
&=
\frac{1}{2}
\E_{c,s}
\left[
\log\frac{\sigma_{\epsilon}^2+\bigl(\alpha + (1-\alpha)c\bigr)^2 V_\alpha}{\sigma_{\epsilon}^2}
\right] \\
&\quad+
\frac{1}{2}
\E_{e,c,s}
\left[
\frac{
\sigma_{\epsilon}^2+\bigl(\alpha + (1-\alpha)c\bigr)^2 (e-\E[e \mid c,s])^2
}{
\sigma_{\epsilon}^2+\bigl(\alpha + (1-\alpha)c\bigr)^2 V_\alpha
}
-1
\right].
\end{align}

Now, conditional on $(c,s)$,
\begin{equation}
\E\left[(e-\E[e \mid c,s])^2 \mid c,s\right]
=
\Var(e \mid c,s)
=
V_\alpha.
\end{equation}
Therefore,
\begin{align} \notag
& \E\left[
\frac{
\sigma_{\epsilon}^2+\bigl(\alpha + (1-\alpha)c\bigr)^2 (e-\E[e \mid c,s])^2
}{
\sigma_{\epsilon}^2+\bigl(\alpha + (1-\alpha)c\bigr)^2 V_\alpha
}
\Bigm\vert c,s
\right] \\
&=
\frac{
\sigma_{\epsilon}^2+\bigl(\alpha + (1-\alpha)c\bigr)^2 V_\alpha
}{
\sigma_{\epsilon}^2+\bigl(\alpha + (1-\alpha)c\bigr)^2 V_\alpha
}
=
1,
\end{align}
so the second expectation vanishes and we have
\begin{equation}
\bar{\risk}^\star
=
\frac{1}{2}
\E_{\p(c,s)}
\left[
\log\left(
1+\frac{\bigl(\alpha + (1-\alpha)c\bigr)^2 V_\alpha}{\sigma_{\epsilon}^2}
\right)
\right].
\end{equation}
Since the integrand depends only on $c$, we get that
\begin{equation}
\bar{\risk}^\star
=
\frac{1}{2}
\E_{\p(c)}
\left[
\log\left(
1+\frac{\bigl(\alpha + (1-\alpha)c\bigr)^2 V_\alpha}{\sigma_{\epsilon}^2}
\right)
\right].
\end{equation}

\section{Experimental details} \label{sup:details}

All experiments were implemented in Python using publicly available datasets. Code will be released upon acceptance.

\subsection{Tradeoff simulation} \label{sup:tradeoff_sim}

For each configuration, we generate $50$ training environments and $50$ test environments, with $1000$ observations per environment, so
each training set contains $50{,}000$ observations and each test set contains $50{,}000$ observations.
All results reported in Figure~\ref{fig:sim} are based on evaluation on the held-out environments and 10 repetitions of the experiment.

For IRM, and VaREx, the predictor is a multilayer perceptron with two hidden layers of width $32$ and ReLU activations, trained with Adam at learning rate $10^{-3}$.
Grid search over $[0.001, 0.01, 0.1, 1.0, 5.0, 10.0, \dots, 90, 100]$ yielded the following penalty coefficients:
\begin{equation}
\lambda_{\mathrm{IRM}} = \lambda_{\mathrm{CLoVE}} = \lambda_{\mathrm{VaREx}} =   0.01.
\end{equation}

For NGMM, we use the same fixed-effects network architecture, together with a scalar random intercept.
The model is fit with the ODE-based marginal-likelihood approximation described in Section~\ref{sec:ri_estimation}.
All methods were trained for 20 epochs with batch size $100$ using Adam with learning rate $10^{-2}$.

\subsection{Colored MNIST} \label{sup:cmnist_details}

In the \emph{causal} regime, we set the baseline values
\begin{equation}
a = 0.75,
\qquad
b = 0.25.
\end{equation}
The resulting values of $(\alpha, \beta, \gamma, \delta)$ for training environments with 
\begin{equation}
\varepsilon \in \{0.20,0.15,0.10\},
\end{equation}
that is,
\begin{equation}
(0.95,0.55,0.45,0.05),
\qquad
(0.90,0.60,0.40,0.10),
\qquad
(0.85,0.65,0.35,0.15),
\end{equation}
and for test environments 
\begin{align*}
& (0.90,0.60,0.40,0.10),
\qquad
(0.55,0.95,0.05,0.45),
\qquad
(0.60,0.90,0.10,0.40), \\
&
\qquad \qquad \qquad (0.65,0.85,0.15,0.35),
\qquad
(0.75,0.75,0.25,0.25).
\end{align*}

All methods share the same architecture:
a multilayer perceptron operating on the flattened $3\times 28 \times 28$ image, with two hidden layers of width $32$ and ReLU activations.
For NGMM, we used the same architecture for the fixed component combined with a learned random-intercept term.

The NGMM model is trained with the Bernoulli ODE-based objective with truncation interval $[-10,10]$ and $32$ ODE steps.

All methods were trained with Adam at learning rate $10^{-3}$ for $1$ epoch and batch size $64$.
Grid search over $[0.001, 0.01, 0.1, 1, 5, 10, \dots, 90, 100]$ yielded the following penalty coefficients:
\begin{equation}
\lambda_{\mathrm{IRM}} = 20,
\qquad
\lambda_{\mathrm{CLoVE}} = 30,
\qquad
\lambda_{\mathrm{VaREx}} = 100.
\end{equation}

\subsection{OGB-MolPCBA} \label{sup:ogb_details}

For all methods we use a small graph neural network with atom and bond encoders, followed by two GINE message-passing layers with hidden width 32, batch normalization, ReLU activations, and global add pooling, and a final linear layer. 

We fit NGMM according to the Bernoulli random-intercept model using the ODE-based objective. 

All methods are trained with Adam, learning rate 0.001, batch size 64, and 10 epochs. 
Grid search for penalty coefficients yielded  
\begin{equation}
\lambda_{\mathrm{IRM}} = \lambda_{\mathrm{CLoVE}} = \lambda_{\mathrm{VaREx}} =  0.1.
\end{equation}
For NGMM we set the truncation interval to $[-3,3]$ and use 64 ODE steps.

\subsection{Camelyon17} \label{sup:camelyon_details}

We resized all images to $96 \times 96$ pixels, centered and normalized. 
For all methods we use the same convolutional network: two convolutional layers with 32 and 64 channels, stride 2, ReLU activations, and adaptive average pooling, followed by a linear classifier. All methods were trained with Adam with learning rate 0.001 for 10 epochs. Grid search for penalty coefficients yielded 
\begin{equation}
\lambda_{\mathrm{IRM}} = 0.85,\qquad
\lambda_{\mathrm{CLoVE}} = 5,\qquad
\lambda_{\mathrm{VaREx}} = 125.
\end{equation}
We fit NGMM with truncation interval of [-3,3] and 64 ODE steps. 

\end{document}